\documentclass[conference]{IEEEtran}
\usepackage{times}

\usepackage{amsmath}
\usepackage{hyperref}
\usepackage{amsfonts}
\usepackage{color}
\usepackage{array}
\usepackage{graphicx}
\graphicspath{ {./figs/} }
\usepackage{mathtools}
\usepackage{amssymb}
\usepackage{cite}
\usepackage{floatrow}
\usepackage[ruled,linesnumbered,noend]{algorithm2e}
\usepackage{wasysym}
\usepackage{todonotes}
\usepackage{cite}
\let\labelindent\relax
\usepackage{enumitem}
\usepackage{chngcntr}
\usepackage{apptools}
\usepackage{stfloats}
\AtAppendix{\counterwithin{lemma}{section}}
\AtAppendix{\counterwithin{assumption}{section}}
\AtAppendix{\counterwithin{theorem}{section}}
\AtAppendix{\counterwithin{corollary}{section}}
\AtAppendix{\counterwithin{definition}{section}}

\newtheorem{theorem}{Theorem}
\newtheorem{problem}{Problem}
\newtheorem{definition}{Definition}
\newtheorem{corollary}{Corollary}
\newtheorem{lemma}{Lemma}

\newcommand{\traj}{\xi}
\newcommand{\state}{x}
\newcommand{\statespace}{\mathcal{X}}
\newcommand{\safeset}{\mathcal{S}}
\newcommand{\unsafeset}{\mathcal{A}}

\newcommand{\numsafe}{N_s}
\newcommand{\numunsafe}{N_{\neg s}}
\newcommand{\control}{u}
\newcommand{\controlset}{\mathcal{U}}

\newcommand{\trajxu}{\traj_{xu}}
\newcommand{\hattrajxu}{\hat{\traj}_{xu}}
\newcommand{\trajx}{\traj_\state}
\newcommand{\traju}{\traj_\control}

\newcommand{\constraintspace}{\mathcal{C}}

\newcommand{\guarunsafe}{\mathcal{G}_{\neg s}}
\newcommand{\guarsafe}{\mathcal{G}_s}
\newcommand{\feas}{\mathcal{F}}
\newcommand{\cstate}{\kappa}

\newcommand{\demj}{\traj_{j}^\textrm{dem}}
\newcommand{\demjopt}{\traj_{j}^*}
\newcommand{\stat}{\textrm{stat}}
\newcommand{\comp}{\textrm{comp}}

\newcommand{\ineq}{\textrm{ineq}}

\newcommand{\struct}{\theta^s }
\newcommand{\params}{\theta^p }
\newcommand{\hatstruct}{\hat{\theta}^s }
\newcommand{\hatparams}{\hat{\theta}^p }

\newcommand{\costparams}{\theta^c }
\newcommand{\hatcostparams}{\hat{\theta}^c }
\newcommand{\Nap}{N_\textrm{AP} }
\newcommand{\ltl}{\varphi(\struct, \params)}
\newcommand{\ltlprop}{\varphi(\hatstruct, \hatparams)}
\newcommand{\Nmin}{N^*}
\newcommand{\eventually}{\lozenge}
\newcommand{\always}{\square}
\newcommand{\andltl}{\wedge}
\newcommand{\orltl}{\vee}

\newcommand{\prop}{p}
\newcommand{\props}{\mathcal{P}}
\newcommand{\feat}{\eta}
\newcommand{\unkfeat}{\mathbf{g}}
\newcommand{\bzero}{\mathbf{0}}

\newcommand{\demunsafe}{\xi^{\neg s}}
\newcommand{\Nsat}{N_\textrm{DAG}}
\newcommand{\Dag}{\mathcal{D}}
\newcommand{\parse}{\mathbf{X}}
\newcommand{\lparse}{\mathbf{L}}
\newcommand{\rparse}{\mathbf{R}}
\newcommand{\satmatrix}{\mathbf{S}}
\newcommand{\Noper}{N_\textrm{g}}
\newcommand{\Ntemp}{N_\textrm{o}}

\newcommand{\Bool}{\mathbf{Z}}
\newcommand{\Booll}{Z}
\newcommand{\seq}{\mathcal{Q}}
\newcommand{\avoid}{\mathcal{V}}
\newcommand{\formset}{\varphi}
\newtheorem{rem}{Remark}
\newtheorem{assumption}{Assumption}

\usepackage[numbers]{natbib}
\usepackage{multicol}

\begin{document}

\title{Explaining Multi-stage Tasks by Learning Temporal Logic Formulas from Suboptimal Demonstrations\vspace{-4pt}}

\author{\authorblockN{Glen Chou, Necmiye Ozay, and Dmitry Berenson}
\authorblockA{Electrical Engineering and Computer Science, University of Michigan, Ann Arbor, MI 48109\\
Email: \texttt{\{gchou, necmiye, dmitryb\}@umich.edu\vspace{-4pt}}}}

\maketitle

\begin{abstract}
We present a method for learning multi-stage tasks from demonstrations by learning the logical structure and atomic propositions of a consistent linear temporal logic (LTL) formula. The learner is given successful but potentially suboptimal demonstrations, where the demonstrator is optimizing a cost function while satisfying the LTL formula, and the cost function is uncertain to the learner. Our algorithm uses the Karush-Kuhn-Tucker (KKT) optimality conditions of the demonstrations together with a counterexample-guided falsification strategy to learn the atomic proposition parameters and logical structure of the LTL formula, respectively. We provide theoretical guarantees on the conservativeness of the recovered atomic proposition sets, as well as completeness in the search for finding an LTL formula consistent with the demonstrations. We evaluate our method on high-dimensional nonlinear systems by learning LTL formulas explaining multi-stage tasks on 7-DOF arm and quadrotor systems and show that it outperforms competing methods for learning LTL formulas from positive examples.
\end{abstract}

\IEEEpeerreviewmaketitle

\vspace{-3pt}
\section{Introduction}
\vspace{-1pt}

Imagine demonstrating a multi-stage task to a robot arm barista, such as preparing a drink for a customer (Fig. \ref{fig:arm_setup}). How should the robot understand and generalize the demonstration? One popular method is inverse reinforcement learning (IRL), which assumes a level of optimality on the demonstrations, and aims to learn a reward function that replicates the demonstrator's behavior when optimized \cite{irl_1, irl_2, lfd3, ng_irl}. Due to this representation, IRL works well on short-horizon tasks, but can struggle to scale to multi-stage, constrained tasks \cite{swirl, Vazquez-Chanlatte18, wafr}. Transferring reward functions across environments (i.e. from one kitchen to another) can also be difficult, as IRL may overfit to aspects of the training environment. It may instead be fruitful to decouple the high- and low-level task structure, learning a logical/temporal abstraction of the task that is valid for different environments which can combine low-level, environment-dependent skills. Linear temporal logic (LTL) is well-suited for representing this abstraction, since it can unambiguously specify high-level temporally-extended constraints \cite{baierkatoen} as a function of atomic propositions (APs), which can be used to describe salient low-level state-space regions. To this end, a growing community in controls and anomaly detection has focused on learning linear temporal logic (LTL) formulas to explain trajectory data. However, the vast majority of these methods require both positive and negative examples in order to regularize the learning problem. While this is acceptable in anomaly detection, where one expects to observe formula-violating trajectories, in the context of robotics, it can be unsafe to ask a demonstrator to execute formula-violating behavior, such as spilling the drink or crashing into obstacles.

\begin{figure}[H]
\centering
\includegraphics[width=9cm]{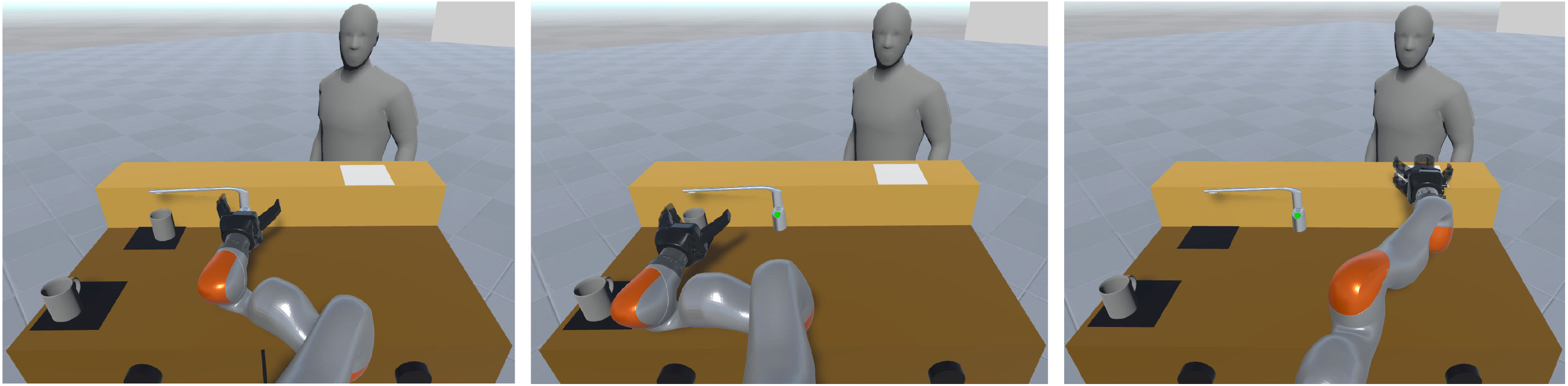}\vspace{-8pt}
\caption{Multi-stage manipulation: first fill the cup, then grasp it, and then deliver it. To avoid spills, a pose constraint is enforced after the cup is grasped.
\vspace{-25pt}}
\label{fig:arm_setup}
\end{figure}
\vspace{-3pt}

In this paper, our insight is that by assuming that demonstrators are goal-directed (i.e. approximately optimize an objective function that may be uncertain to the learner), we can regularize the LTL learning problem without being given formula-violating behavior. In particular, we learn LTL formulas which are parameterized by their high-level logical structure and low-level AP regions, and we show that to do so, it is important to consider demonstration optimality both in terms of the quality of the discrete high-level logical decisions and the continuous low-level control actions. We use the Karush-Kuhn-Tucker (KKT) optimality conditions from continuous optimization to learn the shape of the low-level APs, along with notions of discrete optimality to learn the high-level task structure. We solve a mixed integer linear program (MILP) to jointly recover LTL and cost function parameters which are consistent with the demonstrations. We make the following contributions:

\begin{enumerate}[leftmargin=*]
	\item We develop a method for time-varying, constrained inverse optimal control, where the demonstrator optimizes a cost function while respecting an LTL formula, where the parameters of the atomic propositions, formula structure, and an uncertain cost function are to be learned. We require only positive demonstrations, can handle demonstration suboptimality, and for fixed formula structure, can extract guaranteed conservative estimates of the AP regions.
	\item We develop conditions on demonstrator optimality needed to learn high- and low-level task structure: AP regions can be learned with discrete feasibility, while logical structure requires various levels of discrete optimality. We develop variants of our method under these different assumptions.
	\item We provide theoretical analysis of our method, showing that under mild assumptions, it is guaranteed to return the shortest LTL formula which is consistent with the demonstrations, if one exists. We also prove various results on our method's conservativeness and on formula learnability.
	\item We evaluate our method on learning complex LTL formulas demonstrated on nonlinear, high-dimensional systems, show that we can use demonstrations of the same task on different environments to learn shared high-level task structure, and show that we outperform previous approaches.
\end{enumerate}

\section{Related Work}\label{sec:relatedwork}

There is extensive literature on inferring temporal logic formulas from data via decision trees \cite{BombaraVPYB16}, genetic algorithms \cite{BufoBSBLB14}, and Bayesian inference \cite{Vazquez-Chanlatte18, shah_bayesian}. However, most of these methods require positive and negative examples as input \cite{CamachoM19, KongJAGB14, KongJB17, daniel}, while our method is designed to only use positive examples. Other methods require a space-discretization \cite{VaidyanathanIBD17, ArakiVLVDR19, Vazquez-Chanlatte18}, while our approach learns LTL formulas in the original continuous space. Some methods learn AP parameters, but do not learn logical structure or perform an incomplete search, relying on formula templates \cite{LeungAP19, BakhirkinFM18,  XuNJT19}, while other methods learn structure but not AP parameters \cite{shah_bayesian}. Perhaps the method most similar to ours is \cite{telex}, which learns parametric signal temporal logic (pSTL) formulas from positive examples by fitting formulas that the data tightly satisfies. However, the search over logical structure in \cite{telex} is incomplete, and tightness may not be the most informative metric given goal-directed demonstrations (c.f. Sec. \ref{sec:results}). To our knowledge, this is the first method for learning LTL formula structure and parameters in continuous spaces on high-dimensional systems from only positive examples.

IRL \cite{irl_1, irl_2, boyd, toussaint, pontryagin} searches for a reward function that replicates a demonstrator's behavior when optimized, but these methods can struggle to represent multi-stage, long-horizon tasks \cite{swirl}. To alleviate this, \cite{swirl, RanchodRK15} learn sequences of reward functions, but in contrast to temporal logic, these methods are restricted to learning tasks which can be described by a single fixed sequence. Temporal logic generalizes this, being able to represent tasks that involve more choices and can be completed with multiple different sequences. Some work \cite{PapushaWT18, ZhouL18} aims to learn a reward function given that the demonstrator satisfies a known temporal logic formula; we will learn both jointly.

Finally, there is relevant work in constraint learning. These methods generally focus on learning time-invariant constraints \cite{wafr, corl, ral, lfdc4} or a fixed sequence of task constraints \cite{lfdc1}, which our method subsumes by learning time-dependent constraints that can be satisfied by different sequences.

\section{Preliminaries and Problem Statement}\label{sec:prelims}

We consider discrete-time nonlinear systems $\state_{t+1} = f(\state_t, \control_t, t)$, with state $\state\in\statespace$ and control $\control\in\controlset$, where we denote state/control trajectories of the system as $\trajxu\doteq(\trajx,\traju)$.

We use linear temporal logic (LTL) \cite{baierkatoen}, which augments standard propositional logic to express properties holding on trajectories over (potentially infinite) periods of time. In this paper, we will be given finite-length trajectories demonstrating tasks that can be completed in finite time. To ensure that the formulas we learn can be evaluated on finite trajectories, we focus on learning formulas, given in positive normal form, which are described in a parametric temporal logic similar to bounded LTL \cite{BLTL}, and which can be written with the grammar
\begin{equation}\label{eq:grammar}
	\hspace{-5pt}\varphi ::= p \mid \neg p \mid \varphi_1 \orltl \varphi_2 \mid \varphi_1 \andltl \varphi_2 \mid \always_{[t_1,t_2]} \varphi \mid \varphi_1\ \mathcal{U}_{[t_1,t_2]}\ \varphi_2,\hspace{-1pt}
\end{equation}

\noindent where $\prop \in \props \doteq \{\prop_i\}_{i=1}^{\Nap}$ are atomic propositions (APs) and $\Nap$ is known to the learner. $t_1\leq t_2$ are nonnegative integers. Here, $\neg \varphi$ denotes the negation of formula $\varphi$, $\varphi_1 \vee \varphi_2$ denotes the disjunction of formulas $\varphi_1$ and $\varphi_2$, $\varphi_1 \wedge \varphi_2$ denotes the conjunction of formulas $\varphi_1$ and $\varphi_2$, the ``always" operator $\square_{[t_1, t_2]} \varphi$ denotes that $\varphi$ ``always" has to hold, and the ``until" operator $\varphi_1\ \mathcal{U}_{[t_1, t_2]}\ \varphi_2$ denotes that $\varphi_2$ must eventually hold, and $\varphi_1$ must hold up until that time. The semantics, describing satisfaction of an LTL formula $\varphi$ by a trajectory $\trajxu$, denoted $\varphi \models \trajxu$, are given in App. \ref{sec:app_semantics}. Note that negation only appears directly before APs. Let the size of the grammar be $\Noper = \Nap + \Ntemp$, where $\Ntemp$ is the number of temporal/boolean operators in the grammar. A useful derived operator is ``eventually" $\eventually_{[t_1,t_2]} \varphi \doteq \top\ \mathcal{U}_{[t_1,t_2]}\ \varphi$. We consider tasks that involve optimizing a parametric cost function (encoding efficiency concerns, etc.), while satisfying an LTL formula $\ltl$ (encoding constraints for task completion):

\vspace{2pt}
\begin{problem}[Forward problem]\label{prob:fwdprob}
	\begin{equation*}\label{eq:fwdprob}
	\begin{array}{>{\displaystyle}c >{\displaystyle}l >{\displaystyle}l}
				&\\[-15pt]
		\underset{\trajxu}{\text{minimize}} & \quad c(\trajxu, \costparams) &\\[-1pt]
		\text{subject to} & \quad \trajxu \models \ltl\\[-1pt]
		& \quad \bar\feat(\trajxu) \in \bar\safeset \subseteq \constraintspace\\
	\end{array}\hspace{-15pt}
\end{equation*}
\end{problem}

\noindent where $c(\cdot, \costparams)$ is a potentially non-convex cost function, parameterized by $\costparams \in \Theta^c$. The LTL formula $\ltl$ is parameterized by $\struct \in \Theta^s$, encoding the logical and temporal structure of the formula, and by $\params \doteq \{\params_i\}_{i=1}^{\Nap}$, where $\params_i \in \Theta^p_i$ defines the shape of the region where $\prop_i$ holds. Specifically, we consider APs of the form: $\state \models \prop_i \Leftrightarrow \unkfeat_i(\feat_i(\state), \params_i) \le \bzero$, where $\feat_i(\cdot): \statespace \rightarrow \constraintspace$ is a known nonlinear function, $\unkfeat_i(\cdot,\cdot) \doteq [g_{i,1}(\cdot,\cdot), \ldots, g_{i,{N_i^{\ineq}}}(\cdot,\cdot)]^\top$ is a vector-valued parametric function, and $\constraintspace$ is the space in which the constraint is evaluated, elements of which are denoted \textit{constraint states} $\cstate \in \constraintspace$. In the manipulation example, the joint angles are $\state$, the end effector pose is $\cstate$, and $\feat(\cdot)$ are the forward kinematics. As shorthand, let $G_i(\cstate, \params_i) \doteq \max_{m\in\{1,\ldots,N_i^\ineq\}}\big(g_{i,m}(\cstate, \params_i) \big)$. Define the subset of $\constraintspace$ where $\prop_i$ holds/does not hold, as
\begin{align}
	\safeset_i(\params_i) &\doteq \{ \cstate \mid G_i(\cstate, \params_i) \le 0\}\\
	\unsafeset_i(\params_i) &\doteq \textrm{cl}(\{ \cstate \mid G_i(\cstate, \params_i) > 0\}) = \textrm{cl}(\safeset_i(\params_i)^c)
\end{align}

\noindent To ensure that Problem \ref{prob:fwdprob} admits an optimum, we have defined $\unsafeset_i(\params_i)$ to be closed; that is, states on the boundary of an AP can be considered either inside or outside. For these boundary states, our learning algorithm can automatically detect if the demonstrator intended to visit or avoid the AP (c.f. Sec. \ref{sec:fixedtemplatekkt}). Any \textit{a priori} known constraints are encoded in $\bar\safeset$, where $\bar\feat(\cdot)$ is known. In this paper, we encode in $\bar\safeset$ the system dynamics, start state, and if needed, a goal state separate from the APs. 

We are given $\numsafe$ demonstrations $\{\demj\}_{j=1}^{\numsafe}$ of duration $T_j$, which approximately solve Prob. \ref{prob:fwdprob}, in that they are feasible (satisfy the LTL formula and known constraints) and achieve a possibly suboptimal cost. Note that Prob. \ref{prob:fwdprob} can be modeled with continuous ($\trajxu$) and boolean decision variables (referred to collectively as $\Bool$) \cite{wolff}; the boolean variables determine the high-level plan, constraining the trajectory to obey boolean decisions that satisfy $\ltl$, while the continuous component synthesizes a low-level trajectory implementing the plan. We will use different assumptions of demonstrator optimality on the continuous/boolean parts of the problem, depending on if $\params$ (Sec. \ref{sec:onlyparams}), $\struct$ (Sec. \ref{sec:params_and_struct}), or $\costparams$ (Sec. \ref{sec:costparams}) are being learned, and discuss how these different degrees of optimality can affect the learnability of LTL formulas (Sec. \ref{sec:theory}).

Our goal is to learn the unknown structure $\struct$ and AP parameters $\params$ of the LTL formula $\ltl$, as well as unknown cost function parameters $\costparams$, given demonstrations $\{\demj\}_{j=1}^{\numsafe}$ and the \textit{a priori} known safe set $\bar\safeset$.
\vspace{3pt}
\section{Learning Atomic Proposition Parameters ($\params$)}\label{sec:onlyparams}
\vspace{3pt}

We develop methods for learning unknown AP parameters $\params$ when the cost function parameters $\costparams$ and formula structure $\struct$ are known. We first review recent results \cite{ral} on learning time-invariant constraints via the KKT conditions (Sec. \ref{sec:kkt}), show how the framework can be extended to learn $\params$ (Sec. \ref{sec:fixedtemplatekkt}), and develop a method for extracting states which are guaranteed to satisfy/violate $\prop_i$ (Sec. \ref{sec:volumeextraction}). In this section, we will assume that demonstrations are \textit{locally-optimal for the continuous component} and \textit{feasible for the discrete component}.
\subsection{Learning time-invariant constraints via KKT}\label{sec:kkt}
\vspace{5pt}

Consider a simplified variant of Prob. \ref{prob:fwdprob} that only involves always satisfying a single AP; this reduces Prob. \ref{prob:fwdprob} to a standard trajectory optimization problem:

\vspace{3pt}
\begin{equation}\label{eq:simpfwdprob}
	\begin{array}{>{\displaystyle}c >{\displaystyle}l >{\displaystyle}l}
				&\\[-15pt]
		\underset{\trajxu}{\text{minimize}} & \quad c(\trajxu) &\\
		\text{subject to} & \quad \unkfeat(\feat(\state), \params) \le \bzero, \quad \forall \state \in \trajxu \\
		& \quad \bar\feat(\trajxu) \in \bar\safeset \subseteq \constraintspace\\
	\end{array}\hspace{-15pt}
\end{equation}
\vspace{3pt}

\noindent To ease notation, $\costparams$ is assumed known in Sec. \ref{sec:onlyparams}-\ref{sec:params_and_struct} and reintroduced in Sec. \ref{sec:costparams}. Suppose we rewrite the constraints of \eqref{eq:simpfwdprob} as $\mathbf{h}^k(\feat(\trajxu)) = \mathbf{0}$, $\mathbf{g}^k(\feat(\trajxu)) \le \mathbf{0}$, and $\mathbf{g}^{\neg k}(\feat(\trajxu), \params) \le \mathbf{0}$, where $k/\neg k$ group together known/unknown constraints. Then, with Lagrange multipliers $\lambda$ and $\nu$, the KKT conditions (first-order necessary conditions for local optimality \cite{cvxbook}) of the $j$th demonstration $\demj$, denoted $\textrm{KKT}(\demj)$, are:

\begin{subequations}\label{eq:kkt}
	\small\begin{align}
	\hspace{-25pt}\textrm{Primal}\ \ &\mathbf{h}^{k}(\feat(\state_t^j)) = \mathbf{0},\quad t = 1,\ldots,T_j\label{eq:kkt_primal1}\\
	\hspace{-25pt}\textrm{ feasibility:}\ \ &\mathbf{g}^{k}(\feat(\state_t^j)) \le \mathbf{0},\quad t = 1,\ldots,T_j \label{eq:kkt_primal2}\\
	\hspace{-25pt}  &\mathbf{g}^{\neg k}(\feat(\state_t^j), \textcolor{red}{\params}) \le \mathbf{0},\quad t = 1,\ldots,T_j \label{eq:kkt_primal3}\\[1pt]
	\hline\hspace{-25pt}\textrm{Lagrange mult.}\ \ &\textcolor{blue}{\boldsymbol{\lambda}_t^{j,k}} \ge \mathbf{0},\quad t = 1,\ldots,T_j\label{eq:kkt_lag1}\\
	\textrm{nonnegativity:}\ \ &\textcolor{blue}{\boldsymbol{\lambda}_t^{j,\neg k}} \ge \mathbf{0},\quad t = 1,\ldots,T_j\label{eq:kkt_lag2}\\
	\hline\hspace{-3pt}\textrm{Complementary}\ \ &\textcolor{blue}{\boldsymbol{\lambda}_t^{j,k}}\odot\mathbf{g}^{k}(\feat(\state_t^j)) = \mathbf{0},\ \ t = 1,\ldots,T_j\label{eq:kkt_comp1}\\
	\textrm{slackness:}\ \ &\textcolor{blue}{\boldsymbol{\lambda}_t^{j,\neg k}}\hspace{-1pt}\odot\mathbf{g}^{\neg k}(\feat(\state_t^j), \textcolor{red}{\params}) = \mathbf{0},\ \ t = 1,\ldots,T_j\label{eq:kkt_comp2}\\
	\hline\notag\\[-12pt]
	\hspace{-25pt}\textrm{Stationarity:}\ \ &\nabla_{\state_t} c(\demj) + \textcolor{blue}{\boldsymbol{\lambda}_t^{j,k}}^\top \nabla_{\state_t} \mathbf{g}^{k}(\feat(\state_t^j))\notag
	\\&\quad+\textcolor{blue}{\boldsymbol{\lambda}_t^{j,\neg k}}^{\top} \nabla_{\state_t} \mathbf{g}^{\neg k}(\feat(\state_t^j), \textcolor{red}{\params})\label{eq:kkt_stat}
	\\&\quad+\textcolor{blue}{\boldsymbol{\nu}_t^{j,k}}^\top \nabla_{\state_t} \mathbf{h}^{k}(\feat(\state_t^j)) = \mathbf{0}, \ \ t=1,\ldots,T_j\notag
\end{align}
\end{subequations}
\vspace{4pt}

\noindent where $\odot$ denotes elementwise multiplication. We vectorize the multipliers $\boldsymbol{\lambda}_t^{j,k} \in \mathbb{R}^{N_k^\ineq}$, $\boldsymbol{\lambda}_t^{j,\neg k} \in \mathbb{R}^{N_{\neg k}^\ineq}$, and $\boldsymbol{\nu}_t^{j,k} \in \mathbb{R}^{N_k^\ineq}$, i.e. $\boldsymbol{\lambda}_t^{j,k} = [\lambda_{t,1}^{j,k}, \ldots, \lambda_{t,N_\ineq^k }^{j,k}]^\top $. We drop \eqref{eq:kkt_primal1}-\eqref{eq:kkt_primal2}, as they involve no decision variables. Then, we can find a constraint which makes the $\numsafe$ demonstrations locally-optimal by finding a $\params$ that satisfies the KKT conditions for each demonstration:

\vspace{5pt}
\begin{problem}[KKT, exact]\label{prob:kkt_exact_timeinv}
\begin{equation*}
	\begin{array}{>{\displaystyle}c >{\displaystyle}l >{\displaystyle}l}
				&\\[-15pt]
		\text{find} & \params, \{\boldsymbol{\lambda}_t^{j,k}, \boldsymbol{\lambda}_t^{j,\neg k},\boldsymbol{\nu}_t^{j,k}\}_{t=1}^{T_j},\ \ j = 1,...,\numsafe\\[2pt]
		\text{subject to} & \{\textrm{KKT}(\demj)\}_{j=1}^{\numsafe}\\
	\end{array}
\end{equation*}
\end{problem}
\vspace{3pt}

\noindent If the demonstrations are only approximately locally-optimal, Prob. \ref{prob:kkt_exact_timeinv} may become infeasible. In this case, we can relax stationarity and complementary slackness to cost penalties:

\vspace{3pt}
\begin{problem}[KKT, suboptimal]\label{prob:kkt_subopt}
\begin{equation*}
	\begin{array}{>{\displaystyle}c >{\displaystyle}l >{\displaystyle}l}
				&\\[-15pt]
		\underset{\params, \boldsymbol{\lambda}_t^{j,k}, \boldsymbol{\lambda}_t^{j,\neg k},\boldsymbol{\nu}_t^{j,k}}{\text{minimize}} & \sum_{j=1}^{\numsafe}\big(\Vert \stat(\demj)\Vert_1 + \Vert \comp(\demj)\Vert_1\big)\\
		\text{subject to} & \eqref{eq:kkt_primal3}-\eqref{eq:kkt_lag2}, \ \forall\demj, \ j = 1,\ldots,\numsafe\\[2pt]
	\end{array}\hspace{-15pt}
\end{equation*}
\end{problem}
\vspace{2pt}

\noindent where $\stat(\demj)$ denotes the LHS of Eq. \eqref{eq:kkt_stat} and $\comp(\demj)$ denotes the concatenated LHSs of Eqs. \eqref{eq:kkt_comp1} and \eqref{eq:kkt_comp2}. For some constraint parameterizations (i.e. unions of boxes or ellipsoids \cite{ral}), Prob. \ref{prob:kkt_exact_timeinv}-\ref{prob:kkt_subopt} are MILP-representable and can be efficiently solved; we discuss this in further detail in Sec. \ref{sec:fixedtemplatekkt}.

\setlength{\textfloatsep}{10pt}
\begin{figure}
\centering
\hspace{-5pt}\includegraphics[width=9cm]{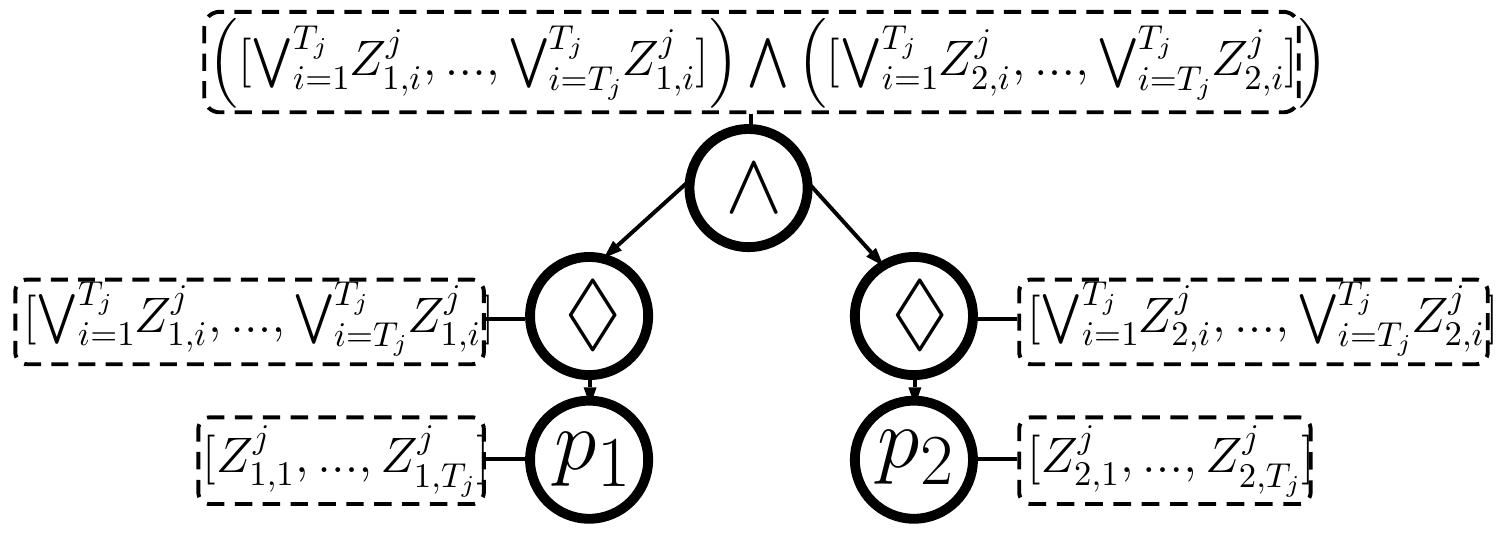}
\caption{A directed acyclic graph (DAG) model of the LTL formula $\varphi = (\eventually_{[0,T_j-1]} p_1)\wedge(\eventually_{[0,T_j-1]} p_2)$ (eventually satisfy $p_1$ and eventually satisfy $p_2$). The DAG representation can be interpreted as a parse tree for $\varphi$ (c.f. Sec. \ref{sec:dag}). The $T_j$ boolean values for each node represent the truth value of the formula associated with the DAG subtree when evaluated on $\demj$, starting at times $t = 1, \ldots, T_j$, respectively.  Each $\demj \models \varphi$ iff the first entry at the root node, $(\bigvee_{i=1}^{T_j} \Booll_{1,i}^j)\bigwedge(\bigvee_{i=1}^{T_j} \Booll_{2,i}^j)$, is true.\vspace{-5pt}}
\label{fig:parse_tree}
\vspace{-10pt}
\end{figure}

\subsection{Modifying KKT for multiple atomic propositions}\label{sec:fixedtemplatekkt}
\vspace{2pt}
Having built intuition with the single AP case, we return to Prob. \ref{prob:fwdprob} and discuss how the KKT conditions change in the multiple-AP setting. We first adjust the primal feasibility condition \eqref{eq:kkt_primal3}. Recall from Sec. \ref{sec:prelims} that we can solve Prob. \ref{prob:fwdprob} by finding a continuous trajectory $\trajxu$ and a set of boolean variables $\Bool$ enforcing that $\trajxu \models \ltl$. For each $\demj$, let $\Bool^j(\params_i) \in \{0,1\}^{\Nap\times T_j}$, and let the $(i,t)$th index $\Booll_{i,t}^{j}(\params_i)$ indicate if on $\demj$, constraint state $\cstate_t \models \prop_i$ for parameters $\params_i$:

\begin{equation} \label{eq:zind}
\begin{gathered}
	\Booll_{i,t}^{j}(\params_i) = 1 \Leftrightarrow \cstate_t \in \safeset_i(\params_i)\\
	\Booll_{i,t}^{j}(\params_i) = 0 \Leftrightarrow \cstate_t \in \unsafeset_i(\params_i)
\end{gathered}
\end{equation}

\noindent Since LTL operators have equivalent boolean encodings \cite{wolff}, the truth value of $\ltl$ can be evaluated as a function of $\Bool^j$, $\params$, and $\struct$, denoted as $\Phi(\Bool^j, \params, \struct)$ (we suppress $\struct$, as it is assumed known for now). For example, consider the LTL formula $\ltl = (\eventually_{[0,T_j-1]} \prop_1) \andltl (\eventually_{[0,T_j-1]} \prop_2)$, which enforces that the system must eventually satisfy $\prop_1$ and eventually satisfy $\prop_2$. We can evaluate the truth value of $\ltl$ $\demj$ by calculating $\Phi(\Bool^j,\params) = (\bigvee_{t=1}^{T_j} \Booll_{1,t}^j(\params_1)) \andltl (\bigvee_{t=1}^{T_j} \Booll_{2,t}^j(\params_2))$ (c.f. Fig. \ref{fig:parse_tree}). Boolean encodings of common temporal and logical operators can be found in \cite{BiereHJLS06}. Enforcing that $\Booll_{i,t}^j(\params_i)$ satisfies \eqref{eq:zind} can be done with a big-M formulation and binary variables $\mathbf{s}_{i,t}^j \in \{0,1\}^{N_i^\ineq}$ \cite{bertsimas}:

\begin{equation}\label{eq:ltlfeas}
\hspace{-10pt}\begin{gathered}
	\unkfeat_i(\cstate_t^j, \params_i) \le M(\mathbf{1}_{N_i^\ineq}-\mathbf{s}_{i,t}^j)\\
	\mathbf{1}_{N_i^\ineq}^\top \mathbf{s}_{i,t}^j - N_i^\ineq \le M\Booll_{i,t}^j - M_\epsilon \\
	\unkfeat_i(\cstate_t^j, \params_i) \ge -M\mathbf{s}_{i,t}^j\\
	\mathbf{1}_{N_i^\ineq}^\top \mathbf{s}_{i,t}^j - N_i^\ineq \ge -M(1-\Booll_{i,t}^j)
\end{gathered}\hspace{-10pt}
\end{equation}

\noindent where $\mathbf{1}_d$ is a $d$-dimensional ones vector, $M$ is a large positive number, and $M_\epsilon \in (0,1)$. In practice, $M$ and $M_\epsilon$ can be carefully chosen to improve the solver's performance. Note that $s_{i,m,t}^j$, the $m$th component of $\mathbf{s}_{i,t}^j$, encodes if $\cstate_t^j$ satisfies a negated $g_{i,m}(\cstate_t^j, \params_i)$, i.e. if $s_{i,m,t}^j = 1$ or $0$, then $\cstate_t^j$ satisfies $g_{i,m}(\cstate_t^j, \params_i) \le$ or $\ge 0$. We can rewrite the enforced constraint as $\mathbf{g}_{i}(\cstate_t^j, \params_i)\odot(2\mathbf{s}_{i,t}^j-\mathbf{1}_{N_i^\ineq}) \le \mathbf{0}$ for each $i$, $t$; we use this form to adapt the remaining KKT conditions. While enforcing \eqref{eq:ltlfeas} is hard in general, if $\mathbf{g}_{i}(\cstate, \params_i)$ is affine in $\params_i$ for fixed $\cstate$, \eqref{eq:ltlfeas} is MILP-representable; henceforth, we assume $\mathbf{g}_{i}(\cstate, \params_i)$ is of this form. Note that this can still describe non-convex regions, as the dependency on $\cstate$ can be nonlinear. To modify complementary slackness \eqref{eq:kkt_comp2} for the multi-AP case, we note that the elementwise product in \eqref{eq:kkt_comp2} is MILP-representable:
\vspace{0pt}
\begin{equation}\label{eq:compslackltl}
\begin{gathered}
	\begin{bmatrix}
		\boldsymbol{\lambda}_{i,t}^{j, \neg k},\ -\mathbf{g}_{i}(\cstate_t^j, \params_i)\odot (2\mathbf{s}_{i,t}^j - \mathbf{1}_{N_i^\ineq})
	\end{bmatrix} \le M\mathbf{Q}_{i,t}^j\\
	\mathbf{Q}_{i,t}^j \mathbf{1}_{2} \le \mathbf{1}_{N_i^\ineq}
\end{gathered}
\end{equation}
where $\mathbf{Q}_{i,t}^j\in \{0,1\}^{N_i^\ineq \times 2}$. Intuitively, \eqref{eq:compslackltl} enforces that either 1) the Lagrange multiplier is zero and the constraint is inactive, i.e. $g_{i,m}(\cstate,\params_i) \in [-M, 0]$ or $\in [0, M]$ if $s_{i,m,t}^j = 0$ or $1$, 2) the Lagrange multiplier is nonzero and $g_{i,m}(\cstate_t, \params_i) = 0$, or both. The stationarity condition \eqref{eq:kkt_stat} must also be modified to consider whether a particular constraint is negated; this can be done by modifying the second line of \eqref{eq:kkt_stat} to terms of the form $\big( {\boldsymbol{\lambda}_{i,t}^{j,\neg k}}^{\top} \hspace{-5pt}\odot (2\mathbf{s}_{i,t}^j - 1) \big) \nabla_{\state_t} \mathbf{g}_i^{\neg k}(\feat(\state_t), \params)$. The KKT conditions for the multi-AP case, denoted $\textrm{KKT}_\textrm{LTL}(\demj)$, then are:

\begin{subequations}\label{eq:kkt_ltl}
	\small\begin{align}
	\hspace{-25pt}\textrm{Primal}\ \ & \textrm{Equations } \eqref{eq:kkt_primal1}-\eqref{eq:kkt_primal2}, \quad t = 1,\ldots,T_j\label{eq:kkt_primal1_ltl}\\
	\hspace{-25pt}\textrm{ feasibility:}\ \ &\textrm{Equation } \eqref{eq:ltlfeas},\quad i = 1,\ldots,\Nap, \ t = 1,\ldots,T_j  \label{eq:kkt_primal2_ltl}\\[1pt]
	\hline\hspace{-25pt}\textrm{Lagrange}\ \ &\textrm{Equation } \eqref{eq:kkt_lag1},\quad t = 1,\ldots,T_j\label{eq:kkt_lag1_ltl}\\
	\textrm{nonneg.:}\ \ &\textcolor{blue}{\boldsymbol{\lambda}_{i,t}^{j,\neg k}} \ge \mathbf{0},\quad i = 1,\ldots,\Nap, \ t = 1,\ldots,T_j\label{eq:kkt_lag2_ltl}\\
	\hline\hspace{-3pt}\textrm{Complem.}\ \ &\textrm{Equation } \eqref{eq:kkt_comp1},\ \ t = 1,\ldots,T_j\label{eq:kkt_comp1_ltl}\\
	\textrm{slackness:}\ \ &\textrm{Equation } \eqref{eq:compslackltl},\quad i = 1,\ldots,\Nap, \ t = 1,\ldots,T_j \label{eq:kkt_comp2_ltl}\\
	\hline\notag\\[-12pt]
	\hspace{-25pt}\textrm{Stationarity:}\ \ &\nabla_{\state_t} c(\demj) + \textcolor{blue}{\boldsymbol{\lambda}_t^{j,k}}^\top \nabla_{\state_t} \mathbf{g}^{k}(\feat(\state_t^j))\notag
	\\[-0pt]&+\hspace{-1pt}\sum_{i=1}^{N_\ineq}\hspace{-1pt}\Big[\textcolor{blue}{\big(\boldsymbol{\lambda}_{i,t}^{j,\neg k}}^{\top}\hspace{-5pt}\odot (2\mathbf{s}_{i,t}^j - 1)\big) \nabla_{\state_t} \mathbf{g}_i^{\neg k}(\feat(\state_t^j), \textcolor{red}{\params_i})\Big]\hspace{-6pt}\label{eq:kkt_stat_ltl}
	\\[-0pt]&+\textcolor{blue}{\boldsymbol{\nu}_t^{j,k}}^\top \nabla_{\state_t} \mathbf{h}^{k}(\feat(\state_t^j)) = \mathbf{0}, \ \ t=1,\ldots,T_j\notag
\end{align}
\end{subequations}

\noindent As mentioned in Sec. \ref{sec:prelims}, if $\cstate_t^j$ lies on the boundary of AP $i$, the KKT conditions will automatically determine if $\cstate_t^j \in \safeset_i(\params_i)$ or $\cstate_t^j \in \unsafeset_i(\params_i)$ based on whichever option enables $\mathbf{s}_{i,t}^j$ to take values that satisfy \eqref{eq:kkt_ltl}. To summarize, our approach is to 1) find $\Bool^j$, which determines the feasibility of $\demj$ for $\ltl$, 2) find $s_{i,m,t}^j$, which link the value of $\Bool^j$ from the AP-containment level (i.e. $\cstate_t^j \in \safeset_i(\params_i)$) to the single-constraint level (i.e. $g_{i,m}(\cstate_t^j, \params_i) \le 0$), and 3) enforce that $\demj$ satisfies the KKT conditions for the continuous optimization problem defined by $\params$ and fixed values of $\mathbf{s}_{i,t}^j$. Finally, we can write the problem of recovering $\params$ for a fixed $\struct$ as:

\vspace{3pt}
\begin{problem}[Fixed template]\label{prob:kkt_exact_ltl}
	\begin{equation*}
	\hspace{-27pt}\begin{array}{>{\displaystyle}c >{\displaystyle}l >{\displaystyle}l}
				&\\[-12pt]
		\text{find} & \params, \boldsymbol{\lambda}_{t}^{j,k}, \boldsymbol{\lambda}_{i,t}^{j,\neg k},\boldsymbol{\nu}_t^{j,k}, \mathbf{s}_{i,t}^j, \mathbf{Q}_{i,t}^{j}, \mathbf{Z}^j,\ \forall i,j,t \\[3pt]
		\text{subject to} & \{\textrm{KKT}_\textrm{LTL}(\demj)\}_{j=1}^{\numsafe}\\[1pt]
	\end{array}\hspace{-20pt}
\end{equation*}
\end{problem}
\vspace{3pt}

We can also encode prior knowledge in Prob. \ref{prob:kkt_exact_ltl}, i.e. known AP labels or a prior on $\params_i$, which we discuss in App. \ref{sec:prior_knowledge}.

\subsection{Extraction of guaranteed learned AP}\label{sec:volumeextraction}

As with the constraint learning problem, the LTL learning problem is also ill-posed: there can be many $\params$ which explain the demonstrations. Despite this, we can measure our confidence in the learned APs by checking if a constraint state $\cstate$ is guaranteed to satisfy/not satisfy $\prop_i$. Denote $\feas_i$ as the feasible set of Prob. \ref{prob:kkt_exact_ltl}, projected onto $\Theta^p_i$ (feasible set of $\params_i$). Then, we say $\cstate$ is learned to be guaranteed contained in/excluded from $\safeset_i(\params_i)$ if for all $\params_i \in \feas_i$, $G_i(\cstate) \le 0\ / \ge 0$. Denote by:

\begin{equation}\label{eq:guarsafe}
		\guarsafe^i\hspace{-3pt} \doteq\hspace{-3pt} \bigcap_{\theta \in \feas_i} \{ \cstate\ |\ G_i(\cstate, \theta) \le 0 \}\hspace{-7pt}
\end{equation}

\begin{equation}\label{eq:guarunsafe}
		\guarunsafe^i\hspace{-3pt} \doteq\hspace{-3pt} \bigcap_{\theta \in \feas_i} \{ \cstate\ |\ G_i(\cstate, \theta) \ge 0 \}\hspace{-7pt}
\end{equation}

\noindent as the sets of $\cstate$ which are guaranteed to satisfy/not satisfy $\prop_i$.

To query if $\cstate$ is guaranteed to satisfy/not satisfy $\prop_i$, we can check the feasibility of the following problem:

\vspace{3pt}
\begin{problem}[Query containment of $\cstate$ in/outside of $\safeset_i(\params_i)$]\label{prob:query}
\begin{equation*}\vspace{1pt}
	\hspace{-27pt}\begin{array}{>{\displaystyle}c >{\displaystyle}l >{\displaystyle}l}
		&\\[-16pt]
		\text{find} & \params, \boldsymbol{\lambda}_{t}^{j,k}, \boldsymbol{\lambda}_{i,t}^{j,\neg k},\boldsymbol{\nu}_t^{j,k}, \mathbf{s}_{i,t}^j, \mathbf{Q}_{i,t}^{j}, \mathbf{Z}^j,\ \forall i,j,t\\[3pt]
		\text{subject to} & \{\textrm{KKT}_\textrm{LTL}(\demj)\}_{j=1}^{\numsafe} \\[3pt]
		& G_i(\cstate, \params_i) \ge 0 \textrm{ \textbf{OR} } G_i(\cstate, \params_i) \le 0
	\end{array}\hspace{-20pt}
\end{equation*}
\end{problem}
\noindent If forcing $\cstate$ to (not) satisfy $\prop_i$ renders Prob. \ref{prob:query} infeasible, we can deduce that to be consistent with the KKT conditions, $\cstate$ must (not) satisfy $\prop_i$. Similarly, continuous volumes of $\cstate$ which must (not) satisfy $\prop_i$ can be extracted by solving:

\vspace{3pt}
\begin{problem}[Volume extraction]\label{prob:volextract}
\begin{equation*}
	\begin{array}{>{\displaystyle}c >{\displaystyle}l >{\displaystyle}l}
			&\\[-15pt]
		\underset{\substack{\varepsilon, \cstate_\textrm{near}, \params, \boldsymbol{\lambda}_t^{j,k}, \boldsymbol{\lambda}_{i,t}^{j,\neg k},\\\boldsymbol{\nu}_t^{j,k}, \mathbf{s}_{i,t}^j, \mathbf{Q}_{i,t}^{j}, \mathbf{Z}^j}}{\text{minimize}} & \varepsilon &\\
		\text{subject to} & \{\textrm{KKT}_\textrm{LTL}(\demj)\}_{j=1}^{\numsafe} \\[3pt]
		& \Vert \cstate_\textrm{near} - \cstate_\textrm{query} \Vert_\infty \le \varepsilon \\[3pt]
		& G_i(\cstate_\textrm{near}, \params_i) > 0 \textrm{ \textbf{OR} } G_i(\cstate_\textrm{near}, \params_i) \le 0\\[2pt]
	\end{array}\hspace{-15pt}
\end{equation*}
\end{problem}
\noindent Prob. \ref{prob:volextract} searches for the largest box centered around $\cstate_\textrm{query}$ contained in $\guarsafe^i$/$\guarunsafe^i$. An explicit approximation of $\guarsafe^i$/$\guarunsafe^i$ can then be obtained by solving Prob. \ref{prob:volextract} for many different $\cstate_\textrm{query}$.

\section{Learning Temporal Logic Structure ($\params$, $\struct$)}\label{sec:params_and_struct}

We will discuss how to frame the search over LTL structures $\struct$ (Sec. \ref{sec:dag}), the learnability of $\struct$ based on demonstration optimality (Sec. \ref{sec:learnability}), and how we combine notions of discrete and continuous optimality to learn $\struct$ and $\params$ (Sec. \ref{sec:genframework}).

\subsection{Representing LTL structure}\label{sec:dag}
We adapt \cite{daniel} to search for a directed acyclic graph (DAG), $\Dag$, that encodes the structure of a parametric LTL formula and is equivalent to its parse tree, with identical subtrees merged. Hence, each node still has at most two children, but can have multiple parents. This framework enables both a complete search over length-bounded LTL formulas and encoding of specific formula templates through constraints on $\Dag$ \cite{daniel}.

Each node in $\Dag$ is labeled with an AP or operator from \eqref{eq:grammar} and has at most two children; binary operators like $\wedge$ and $\vee$ have two, unary operators like $\eventually_{[t_1,t_2]}$ have one, and APs have none (see Fig. \ref{fig:parse_tree}). Formally, a DAG with $\Nsat$ nodes, $\Dag = (\parse, \lparse, \rparse)$, can be represented as: $\parse\in \{0, 1\}^{\Nsat \times \Noper}$, where $\parse_{u,v} = 1$ if node $u$ is labeled with element $v$ of the grammar and $0$ else, and $\lparse, \rparse\in \{0,1\}^{\Nsat\times\Nsat}$, where $\lparse_{u,v}=1$ / $\rparse_{u,v}=1$ if node $v$ is the left/right child of node $u$ and $0$ else. The DAG is enforced to be well-formed (i.e. there is one root node, no isolated nodes, etc.) with further constraints; see \cite{daniel} for more details. Since $\Dag$ defines a parametric LTL formula, we set $\struct = \Dag$. To ensure that demonstration $j$ satisfies the LTL formula encoded by $\Dag$, we introduce a satisfaction matrix $\satmatrix_j^\textrm{dem} \in \{0,1\}^{\Nsat\times T_j}$, where $\satmatrix_{j,(u,t)}^\textrm{dem}$ encodes the truth value of the subformula for the subgraph with root node $u$ at time $t$ (i.e., $\satmatrix_{j,(u,t)}^\textrm{dem}=1$ iff the suffix of $\demj$ starting at time $t$ satisfies the subformula). This can be encoded with constraints:

\vspace{-5pt}
\begin{equation}\label{eq:sat1}
	|\satmatrix_{j,(u,t)}^\textrm{dem} - \Phi_{uv}^t| \le M(1 - \parse_{u,v})
\end{equation}
\vspace{-10pt}

\noindent where $\Phi_{uv}^t$ is the truth value of the subformula for the subgraph rooted at $u$ if labeled with $v$, evaluated on the suffix of $\demj$ starting at time $t$. The truth values are recursively generated, and the leaf nodes, each labeled with some AP $i$, have truth values set to $\Bool_i^j(\params_i)$. Next, we can enforce that the demonstrations satisfy the formula encoded in $\Dag$ by enforcing:

\vspace{-7pt}
\begin{equation}\label{eq:sat2}
	\satmatrix_{j,(\textrm{root},1)}^\textrm{dem} = 1,\ j = 1,\ldots, \numsafe
\end{equation}
\vspace{-12pt}

\noindent We will also use synthetically-generated invalid trajectories $\{\traj^{\neg s}\}_{j=1}^{\numunsafe}$ (Sec. \ref{sec:genframework}). To ensure $\{\traj^{\neg s}\}_{j=1}^{\numunsafe}$ do not satisfy the formula, we add more satisfaction matrices $\satmatrix_j^{\neg s}$ and enforce:

\vspace{-9pt}
\begin{equation}\label{eq:sat3}
	 \satmatrix_{j,(\textrm{root},1)}^{\neg s} =0, \ j = 1,\ldots, \numunsafe.
\end{equation}
\vspace{-13pt}

\noindent After discussing learnability, we will show how $\Dag$ can be integrated into the KKT-based learning framework in Sec. \ref{sec:genframework}.

\subsection{A detour on learnability}\label{sec:learnability}
When learning only the AP parameters $\params$ (Sec. \ref{sec:onlyparams}), we assumed that the demonstrator chooses any \textit{feasible} assignment of $\Bool$ consistent with the specification, then finds a locally-optimal trajectory for those fixed $\Bool$. Feasibility is enough if the structure $\struct$ of $\ltl$ is known: to recover $\params$, we just need to find some $\Bool$ which is feasible with respect to the known $\struct$ (i.e. $\Phi(\Bool^j, \params, \struct) = 1$) and makes $\demj$ locally-optimal; that is, the demonstrator can choose an arbitrarily suboptimal high-level plan as long as its low-level plan is locally-optimal for the chosen high-level plan. However, if $\struct$ is also unknown, only using boolean feasibility is not enough to recover meaningful logical structure, as this makes any formula $\varphi$ for which $\Phi(\Bool^j, \params, \struct) = 1$ consistent with the demonstration, including trivially feasible formulas always evaluating to $\top$. Consider the example in Fig. \ref{fig:venn_diagram}: $\params_1, \params_2$ are known and we are given two kinematic demonstrations minimizing path length under input constraints, formula $\varphi = (\neg p_2\ \mathcal{U}_{[0, T_j-1]}\ p_1)\wedge \eventually_{[0,T_j-1]} p_2 $, and start/goal constraints.
Assuming boolean feasibility, we cannot distinguish between formulas in $\formset_f$, the set of formulas for which the demonstrations are feasible in the discrete variables and locally-optimal in the continuous variables.

On the other end of the spectrum, we can assume the demonstrator is \textit{globally-optimal}. This invalidates many structures in $\formset_f$, i.e. the blue trajectory should not visit both $\safeset_1$ and $\safeset_2$ if $\varphi = (\eventually_{[0,T_j-1]} p_1)\vee(\eventually_{[0,T_j-1]} p_2)$; we achieve a lower cost by only visiting one. Using global optimality, we can distinguish between all but the formulas with globally-optimal trajectories of equal cost (formulas in $\formset_g$), i.e. we cannot learn the ordering constraint ($\neg p_2\ \mathcal{U}_{[0,T_j-1]}\ p_1$) from only the blue trajectory, as it coincides with the globally-optimal trajectory for $\varphi = (\eventually_{[0,T_j-1]} p_1) \andltl (\eventually_{[0,T_j-1]} p_2)$; we need the yellow trajectory to distinguish the two. We now define an optimality condition between feasibility and global optimality
\begin{definition}[Spec-optimality]\label{def:specopt}
	A demonstration $\demj$ is \textit{$\mu$-spec-optimal ($\mu$-SO)}, where $\mu \in \mathbb{Z}_+$, if for every index set $\iota \doteq \{(i_1,t_1),...,(i_\mu,t_\mu) \}$ in $\mathcal{I} \doteq \{ \iota \mid i_m \in \{1,...,\Nap\}, t_m \in \{1,..., T_j\}, m=1,...,\mu\}$, at least one of the following holds:
	\begin{itemize}
		\item $\demj$ is locally-optimal after removing the constraints associated with $\prop_{i_m}$ on $\cstate_{t_m}^j$, for all $(i_m, t_m) \in \iota$.
		\item For each index $(i_m, t_m) \in \iota$, the formula is not satisfied for a perturbed $\Bool$, denoted $\hat\Bool$, where $\hat\Booll_{i_m,t_m}(\params_{i_m}) = \neg \Booll_{i_m,t_m}(\params_{i_m})$, for all $m=1,\ldots,\mu$, and $\hat\Booll_{i',t'}(\params_{i'}) = \Booll_{i',t'}(\params_{i'})$ for all $(i',t')\notin \iota$.
		\item $\demj$ is infeasible with respect to $\hat\Bool$.
	\end{itemize}
\end{definition}

Spec-optimality enforces a level of logical optimality: if a state $\cstate_t^j$ on demonstration $\demj$ lies inside/outside of AP $i$ (i.e. $G_i(\cstate_t^j, \params_i) \le 0$ / $\ge 0$), and the cost $c(\demj)$ can be lowered if that AP constraint is relaxed, then the constraint must hold to satisfy the specification. Intuitively, this means that the demonstrator does not visit/avoid APs which will needlessly increase the cost and are not needed to complete the task. Further discussion regarding spec-optimality is presented in App. \ref{sec:spec_opt}. As globally-optimal demonstrations must also be spec-optimal (c.f. Lem. \ref{lem:specglobopt}), we will use spec-optimality to vastly reduce the search space when searching for formulas which make the demonstrations globally-optimal (Sec. \ref{sec:genframework}).

\begin{figure}[h]
\centering
\includegraphics[width=\linewidth]{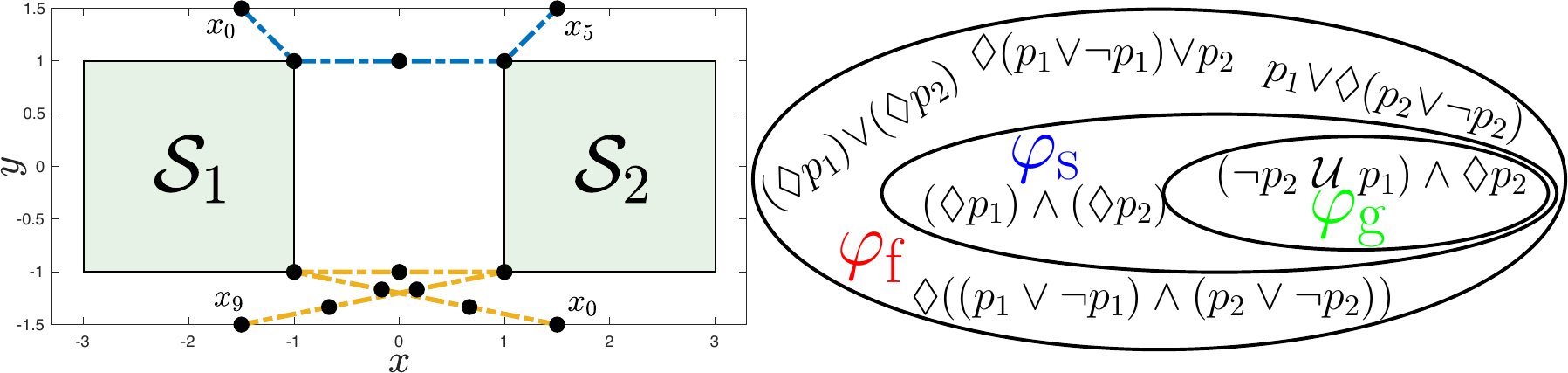}
\caption{\textbf{Left}: Two demonstrations which satisfy the LTL formula $\varphi = (\neg p_2\ \mathcal{U}_{[0,T_j-1]}\ p_1)\wedge \eventually_{[0,T_j-1]} p_2 $ (first satisfy $p_1$, then satisfy $p_2$). \textbf{Right}: Some example formulas that are consistent with $\varphi$, for various levels of discrete optimality ($\varphi_f$: discrete feasibility, $\varphi_s$: spec-optimality, $\varphi_g$: discrete global optimality).}\label{fig:venn_diagram}
\end{figure}

\begin{algorithm}
\SetInd{0.25em}{1em}
\SetAlgoLined
\textbf{Input}: $\{\demj\}_{j=1}^{\numsafe}$, $\bar\safeset$, \textbf{Output}: $\hatstruct, \hatparams$\\
$\Nsat \leftarrow 0$, $\{\demunsafe\} \leftarrow \{\}$

\While{$\neg$ consistent}{
$\Nsat \leftarrow \Nsat + 1$\\

\While{Problem \ref{prob:globopt} is feasible}{
$\hatstruct$, $\hatparams$ $\leftarrow$ Problem \ref{prob:globopt}$(\{\demj\}_{j=1}^{\numsafe}, \{\demunsafe\}, \Nsat)$ \\
\For{$j = 1$ to $\numsafe$}{
$\trajxu^j \leftarrow$ Problem \ref{prob:fwdprob_ctrexample}$(\demj)$ \\
\lIf{\normalfont$c(\trajxu^j) < \frac{c(\demj)}{(1+\delta)}$}{
$\{\demunsafe\} \leftarrow \{\demunsafe\} \cup \trajxu$
}
}
\lIf{\normalfont$\bigvee_{j=1}^{\numsafe} (c(\trajxu^j) < \frac{c(\demj)}{(1+\delta)})$}{
consistent $\leftarrow \top$; break}}}
 \caption{Falsification}\label{alg:falsification}
\end{algorithm}

\subsection{Counterexample-guided framework}\label{sec:genframework}

In this section, we will assume that the demonstrator returns a solution to Prob. \ref{prob:fwdprob} which is \textit{boundedly-suboptimal with respect to the globally optimal solution}, in that $c(\demj) \le (1+\delta)c(\demjopt)$, for a known $\delta$, where $c(\demjopt)$ is the cost of the optimal solution. This is reasonable as the demonstration should be feasible (completes the task), but may be suboptimal in terms of cost (i.e. path length, etc.), and $\delta$ can be estimated from repeated demonstrations. Under this assumption, any trajectory $\trajxu$ satisfying the known constraints $\bar\feat(\trajxu) \in \bar \safeset$ at a cost lower than the suboptimality bound, i.e. $c(\trajxu) \le c(\demj)/(1+\delta)$, must violate $\ltl$ \cite{wafr, corl}. We can use this to reject candidate structures $\hatstruct$ and parameters $\hatparams$. If we can find a counterexample trajectory that satisfies the candidate LTL formula $\ltlprop$ at a lower cost by solving Prob. \ref{prob:fwdprob_ctrexample},

\vspace{2pt}
\begin{problem}[Counterexample search]\label{prob:fwdprob_ctrexample}
	\begin{equation*}
	\begin{array}{>{\displaystyle}c >{\displaystyle}l >{\displaystyle}l}
				&\\[-19pt]
		\text{find} & \quad \trajxu &\\[-2pt]
		\text{subject to} & \quad \trajxu \models \ltlprop \\
		& \quad \bar\feat(\trajxu) \in \bar\safeset(\demj) \subseteq \constraintspace\\
		& \quad c(\trajxu) < c(\demj)/(1+\delta)
	\end{array}\hspace{-15pt}
\end{equation*}
\end{problem}
\vspace{-2pt}

\noindent then $\ltlprop$ cannot be consistent with the demonstration. Thus, we can search for a consistent $\hatstruct$ and $\hatparams$ by iteratively proposing candidate $\hatstruct$ / $\hatparams$ by solving Prob. \ref{prob:globopt} (a modified version of Prob. \ref{prob:kkt_exact_ltl}, which we will discuss shortly) and searching for counterexamples that can prove the parameters are invalid/valid; this is summarized in Alg. \ref{alg:falsification}. Heuristics on the falsification loop are discussed in App. \ref{sec:loopvariants}. We now discuss the core components of Alg. \ref{alg:falsification} (Probs. \ref{prob:fwdprob_ctrexample} and \ref{prob:globopt}) in detail.

\noindent\textbf{Counterexample generation}: We propose different methods to solve Prob. \ref{prob:fwdprob_ctrexample} based on the dynamics. For piecewise affine systems, Prob. \ref{prob:fwdprob_ctrexample} can be solved directly as a MILP \cite{wolff}. However, the LTL planning problem for general nonlinear systems is challenging \cite{LiF17, FuPT17}. Probabilistically-complete sampling-based methods \cite{LiF17, FuPT17} or falsification tools \cite{AnnpureddyLFS11} can be applied, but can be slow on high-dimensional systems. For simplicity and speed, we solve Prob. \ref{prob:fwdprob_ctrexample} by finding a trajectory $\hattrajxu \models \ltlprop$ and boolean assignment $\Bool$ for a kinematic approximation of the dynamics via solving a MILP, then warm-start the nonlinear optimizer with $\hattrajxu$ and constrain it to be consistent with $\Bool$, returning some $\trajxu$. If $c(\trajxu) < c(\demj)/(1+\delta)$, then we return, otherwise, we generate a new $\hattrajxu$. Whether this method returns a valid counterexample depends on if the nonlinear optimizer converges to a feasible solution; hence, this approach is not complete. However, we show that it works well in practice (see Sec. \ref{sec:results}).

\noindent \textbf{Unifying parameter and structure search}: When both $\params$ and $\struct$ are unknown, they must be jointly learned due to their interdependence: learning the structure involves finding an unknown boolean function of $\params$, parameterized by $\struct$, while learning the AP parameters $\params$ requires knowing which APs were selected or negated, determined by $\struct$. This can be done by combining the KKT \eqref{eq:kkt_ltl} and DAG constraints \eqref{eq:sat1}-\eqref{eq:sat3} into a single MILP, which can then be integrated into Alg. \ref{alg:falsification}:

\vspace{3pt}
\begin{problem}[Learning $\params$, $\struct$ by global optimality, KKT]\label{prob:globopt}
\begin{equation*}\vspace{2pt}
	\hspace{-20pt}\begin{array}{>{\displaystyle}c >{\displaystyle}l >{\displaystyle}l}
				&\\[-15pt]
		\text{find} & \begin{split}\Dag, \satmatrix_j^\textrm{dem}, \satmatrix_j^{\neg s}, \params, \boldsymbol{\lambda}_{t}^{j,k}, \boldsymbol{\lambda}_{i,t}^{j,\neg k},\boldsymbol{\nu}_t^{j,k}, \mathbf{s}_{i,t}^j, \mathbf{Q}_{i,t}^{j}, \mathbf{Z}^j, \\ \forall i,j,t\end{split} \\
		\text{s.t.} & \{\textrm{KKT}_\textrm{LTL}(\demj)\}_{j=1}^{\numsafe}\\[1pt]
		& \textrm{topology constraints for } \Dag \\
		& \textrm{Equations }\eqref{eq:sat1}-\eqref{eq:sat2},\ j = 1,\ldots,\numsafe \\
		& \textrm{Equation }\eqref{eq:sat3},\ j = 1,\ldots,\numunsafe \\
	\end{array}\hspace{-20pt}
\end{equation*}
\end{problem}
\vspace{-5pt}

In Prob. \ref{prob:globopt}, since 1) the $\Bool_i^j(\params_i)$ at the leaf nodes of $\Dag$ are constrained via \eqref{eq:ltlfeas} to be consistent with $\params$ and $\demj$ and 2) the formula defined by $\Dag$ is constrained to be satisfied for the $\Bool$ via \eqref{eq:sat1}, the low-level demonstration $\demj$ must be feasible for the overall LTL formula defined by the DAG, i.e. $\ltl$, where $\struct = \Dag$.
$\textrm{KKT}_\textrm{LTL}(\demj)$ then chooses AP parameters $\params$ to make $\demj$ locally-optimal for the continuous optimization induced by a fixed realization of boolean variables. Overall, Prob. \ref{prob:globopt} finds a pair of $\params$ and $\struct$ which makes $\demj$ locally-optimal for a fixed $\Bool^j$ which is \textit{feasible} for $\ltl$, i.e. $\Phi(\Bool^j, \params, \struct) = 1$, for all $j$. To also impose the spec-optimality conditions (Def. \ref{def:specopt}), we can add these constraints to Prob. \ref{prob:globopt}: 

\begin{subequations}\label{eq:specopt}
	\begin{gather}
		 \satmatrix_{j,(\textrm{root},1)}^{\textrm{dem}, \hat \Bool_n^j} \le b_{nj}^1 \label{eq:specopt_1}\\
		\begin{split}\Vert {\boldsymbol{\lambda}_{i_m,t_m}^{j,\neg k}}^{\hspace{-8pt}\top} \nabla_{\state_t} \mathbf{g}_{i_m}^{\neg k}(\feat(\state_t^j), \params_{i_m}) \Vert \le M(1-b_{nj}^2),\\ m = 1,...,\mu\hspace{-5pt}\end{split}\label{eq:specopt_2}\\
		\mathbf{g}_{i_m}^{\neg k}(\eta(\state_t^j), \params_{i_m}) \ge -M(1- \mathbf{e}_{nm}^j),\ m = 1,...,\mu\label{eq:specopt_3} \\
		\mathbf{1}_{N_\ineq^{i_m}}^\top \mathbf{e}_{nm}^j \ge \hat \Booll_{i_m t_m}^j(\params_{i_m}) - b_{nj}^3,\ m = 1,...,\mu\label{eq:specopt_4} \\
		\mathbf{g}_{i_m}^{\neg k}(\eta(\state_t^j), \params_{i_m}) \le M(\hat \Booll_{i_m, t_m}^j + b_{nj}^3)\label{eq:specopt_5}\\[3pt]
		\begin{split}b_{nj}^1 + b_{nj}^2 + b_{nj}^3 \le 1,\quad \mathbf{b}_{nj} \in \{0,1\}^3,\\ \mathbf{e}_{nm}^j \in \{0,1\}^{N_\ineq^{i_m}}\end{split}\label{eq:specopt_6}
	\end{gather}
\end{subequations}

\noindent for $n = 1,\ldots,|\mathcal{I}|$, where $\satmatrix_{j}^{\textrm{dem}, \hat \Booll_n^j}$ is the satisfaction matrix for $\demj$ where the leaf nodes are perturbed to take the values of $\hat \Booll_n^j$, where $n$ indexes an $\iota \in \mathcal{I}$. \eqref{eq:specopt_1} models the case when the formula is not satisfied, \eqref{eq:specopt_2} models when $\demj$ remains locally-optimal upon relaxing the constraint (zero stationarity contribution), and \eqref{eq:specopt_3}-\eqref{eq:specopt_5} model the infeasible case.

\begin{rem}
	If $\mu = 1$, the infeasibility constraints \eqref{eq:specopt_3}-\eqref{eq:specopt_5} can be ignored (since together with \eqref{eq:specopt_1}, they are redundant), and we can modify \eqref{eq:specopt_6} to $b_{nj}^1 + b_{nj}^2 \le 1$, $\mathbf{b}_{nj} \in \{0,1\}^2$.
\end{rem}
\begin{rem}
It is only useful to enforce spec-optimality on index pairs $(i_1,t_1),\ldots, (i_\mu,t_\mu)$ where $G_{i_m}(\cstate_{t_m}^j, \params_{i_m}) = 0$ for all $m = 1,...,\mu$; otherwise the infeasibility case automatically holds. If $\params$ is unknown, we won't know \textit{a priori} when this holds, but if $\params$ are (approximately) known, we can pre-process so that spec-optimality is only enforced for salient $\iota \in \mathcal{I}$.
\end{rem}
\begin{rem}\label{rem:surveillance}
	Prob. \ref{prob:globopt} with spec-optimality constraints \eqref{eq:specopt} can be used to directly search for a $\ltlprop$ which can be satisfied by visiting a set of APs in any order (i.e. surveillance-type tasks) without using the loop in Alg. \ref{alg:falsification}, since \eqref{eq:specopt} directly enforces that any AP (1-SO) or a set of APs ($\mu$-SO) which were visited and which prevent the trajectory cost from being lowered must be visited for any candidate $\ltlprop$.
\end{rem}

\section{Learning Cost Function Parameters ($\params$, $\struct$, $\costparams$)}\label{sec:costparams}

If $\costparams$ is unknown, it can be learned by modifying $\textrm{KKT}_\textrm{LTL}$ to also consider $\costparams$ in the stationarity condition: all terms like $\nabla_{\trajxu} c(\demj)$ should be modified to $\nabla_{\trajxu} c(\demj, \costparams)$. When $c(\cdot, \cdot)$ is affine in $\costparams$ for fixed $\demj$, the stationarity condition is representable with a MILP constraint. However, the falsification loop in Alg. \ref{alg:falsification} requires a fixed cost function in order to judge if a trajectory is a counterexample. Thus, one valid approach is to first solve Prob. \ref{prob:globopt}, searching also for $\costparams$, then fixing $\costparams$, and running Alg. \ref{alg:falsification} for the fixed $\costparams$ (see App. \ref{sec:app_costalg}). Note that this procedure either eventually returns an LTL formula consistent with the fixed $\costparams$, or Alg. \ref{alg:falsification} becomes infeasible, and a new $\costparams$ must be generated and Alg. \ref{alg:falsification} rerun. While this procedure is guaranteed to eventually return a set of $\costparams$, $\struct$, and $\params$ which make each $\demj$ globally-optimal with respect to $c(\trajxu, \costparams)$ under $\ltl$, it may require iterating through an infinite number of candidate $\costparams$ and hence is not guaranteed to terminate in finite time (Cor. \ref{thm:completeness_costunc}). Nevertheless, we note that for certain simple classes of formulas (Rem. \ref{rem:surveillance}), a consistent set of $\costparams$, $\struct$, and $\params$ can be recovered in one shot.

\section{Theoretical Analysis}\label{sec:theory}

In this section, we prove that our method is complete under some assumptions, without (Thm. \ref{thm:completeness}) or with (Cor. \ref{thm:specoptcompleteness}) spec-optimality, and that we can compute guaranteed conservative estimates of $\safeset_i$/$\unsafeset_i$ (Thm. \ref{thm:innerapprox}). Finally, we show that leveraging stronger optimality assumptions on the demonstrator shrinks the set of consistent formulas (Thm. \ref{thm:distinguishability}). See App. \ref{sec:app_theory} for proofs.
\begin{assumption}\label{ass:complete}
	Prob. \ref{prob:fwdprob_ctrexample} is solved with a complete planner.
\end{assumption}
\begin{assumption}\label{ass:locopt}
	Each demonstration is locally-optimal (i.e. satisfies the KKT conditions) for fixed boolean variables.
\end{assumption}
\begin{assumption}\label{ass:feasible}
	The true parameters $\params$, $\struct$, and $\costparams$ are in the hypothesis space of Prob. \ref{prob:globopt}: $\params \in \Theta_p$, $\struct \in \Theta_s$, $\costparams \in \Theta_c$.
\end{assumption}

\begin{theorem}[Completeness and consistency, unknown case]\label{thm:completeness}
	Under Assumptions \ref{ass:complete}-\ref{ass:feasible}, Alg. \ref{alg:falsification} is guaranteed to return a formula $\ltl$ such that 1) $\demj \models \ltl$ and 2) $\demj$ is globally-optimal under $\ltl$, for all $j$, 3) if such a formula exists and is representable by the provided grammar.
\end{theorem}

\begin{corollary}[Shortest formula]
	Let $\Nmin$ be the minimal size DAG for which there exists $(\params, \struct)$ such that $\demj \models \ltl$ for all $j$. Under Assumptions \ref{ass:complete}-\ref{ass:feasible}, Alg. \ref{alg:falsification} is guaranteed to return a DAG of length $\Nmin$.
\end{corollary}
\begin{lemma}\label{lem:specglobopt}
	All globally-optimal trajectories are $\mu$-SO.
\end{lemma}
\begin{corollary}[Alg. \ref{alg:falsification} with spec-optimality]\label{thm:specoptcompleteness}
	By modifying Alg. \ref{alg:falsification} so that Prob. \ref{prob:globopt} uses constraints \eqref{eq:specopt}, Alg. \ref{alg:falsification} still returns a consistent solution $\ltlprop$ if one exists, i.e. each $\demj$ is feasible and globally optimal for each $\ltlprop$.
\end{corollary}
\begin{corollary}[Completeness and consistency, unknown cost]\label{thm:completeness_costunc}
	Under Assumptions \ref{ass:complete}-\ref{ass:feasible}, Alg. \ref{alg:falsification_cost} returns a formula $\ltl$ such that 1) $\demj \models \ltl$ and 2) $\demj$ is globally-optimal with respect to $\costparams$ under the constraints of $\ltl$, for all $j$, 3) if such a formula exists and is representable by the provided grammar.
\end{corollary}
\begin{theorem}[Conservativeness for fixed template]\label{thm:innerapprox}
	Suppose that $\struct$ and $\costparams$ are known, and $\params$ is unknown. Then, extracting $\guarsafe^i$ and $\guarunsafe^i$, as defined in \eqref{eq:guarsafe}-\eqref{eq:guarunsafe}, from the feasible set of Prob. \ref{prob:kkt_exact_ltl} projected onto $\Theta_i^p$ (denoted $\feas_i$), returns $\guarsafe^i \subseteq \safeset_i$ and $\guarunsafe^i \subseteq \unsafeset_i$, for all $i \in \{1, \ldots, \Nap\}$.
\end{theorem}
\begin{theorem}[Distinguishability]\label{thm:distinguishability}
	For the consistent formula sets defined in Sec. \ref{sec:learnability}, we have $\formset_{g} \subseteq \formset_{\tilde\mu\textrm{-SO}} \subseteq \formset_{\hat\mu\textrm{-SO}} \subseteq \formset_f$, for $\tilde\mu > \hat\mu$.
\end{theorem}

\section{Experimental Results}\label{sec:results}

We show that our algorithm outperforms a competing method, can learn shared task structure from demonstrations across environments, and can learn LTL formulas $\params$, $\struct$ and uncertain cost functions $\costparams$ on high-dimensional problems. Please refer to the supplementary video for visualizations of the results: \textbf{\textcolor{blue}{\url{https://youtu.be/cpUEcWCUMqc}}}.

\begin{figure}[h]
\centering
\vspace{-5pt}
\includegraphics[width=\linewidth]{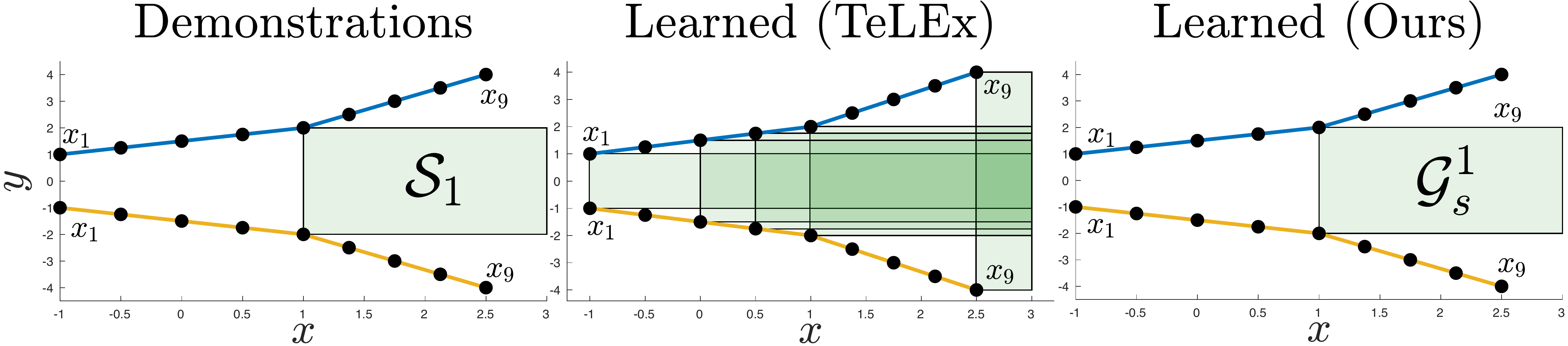}
\caption{Toy example for baseline comparison \cite{telex}.}\label{fig:telex}
\vspace{-5pt}
\end{figure}

\noindent \textbf{Baseline comparison}:
Likely the closest method to ours is \cite{telex}, which learns a pSTL formula that is tightly satisfied by the demonstrations via solving a nonconvex problem to local optimality: $\arg \max_{\params} \min_{j} \tau(\params, \demj)$, where $\tau(\params, \demj)$ measures how tightly $\demj$ fits the learned formula. We run the authors' code \cite{jha} on a toy problem (see Fig. \ref{fig:telex}), where the demonstrator has kinematic constraints, minimizes path length, and satisfies start/goal constraints and $\varphi = \eventually_{[0,8]} \prop_1$, where $\state \models \prop_1 \Leftrightarrow [I_{2\times 2}, -I_{2\times 2}]^\top \state \le [3, 2, -1, 2]^\top={[3, \params_1]}^\top$. We assume the structure $\struct$ is known, and we aim to learn $\params$ to explain why the demonstrator deviated from an optimal straight-line path to the goal. Solving Prob. \ref{prob:volextract} returns $\guarsafe^1 = \safeset_1$ (Fig. \ref{fig:telex}, right). On the other hand, we run TeLEx multiple times, converging to different local optima, each corresponding to a ``tight" $\params$ (Fig. \ref{fig:telex}, center): TeLEx cannot distinguish between multiple different ``tight" $\params$, which makes sense, as the method tries to find \textit{any} ``tight" solution. This example suggests that if the demonstrations are goal-directed, a method that leverages their optimality is likely to better explain them.

\noindent \textbf{Learning shared task structure}:
In this example, we show that our method can extract logical structure shared between demonstrations that complete the same high-level task, but in different environments (Fig. \ref{fig:transfer}). A point robot must first go to the mug ($\prop_1$), then go to the coffee machine ($\prop_2$), and then go to goal ($\prop_3$) while avoiding obstacles ($\prop_4,\prop_5$). As the floor maps differ, $\params$ also differ, and are assumed known. We add two relevant primitives to the grammar, sequence: $\varphi_1\ \seq\ \varphi_2 \doteq \neg \varphi_2\ \mathcal{U}_{[0, T_j-1]}\ \varphi_1$, enforcing that $\varphi_2$ cannot occur until after $\varphi_1$ has occurred for the first time, and avoid: $\avoid \varphi \doteq \always_{[0, T_j-1]} \neg \varphi$, enforcing $\varphi$ never holds over $[1, T_j]$. Then, the true formula is: $\varphi^* = \avoid \prop_4 \wedge \avoid \prop_5 \wedge (\prop_1\ \seq\ \prop_2) \wedge (\prop_2\ \seq\ \prop_3) \wedge \eventually_{[0,T_j-1]} \prop_3$.

Suppose first that we are given the blue demonstration in Env. 2. Running Alg. \ref{alg:falsification} with 1-SO constraints \eqref{eq:specopt} terminates in one iteration at $\Nsat = 14$ with $\varphi_0 = \avoid \prop_4 \wedge \avoid \prop_5 \wedge \eventually_{[0, T_j-1]} \prop_2 \wedge \eventually_{[0, T_j-1]} \prop_3 \wedge (\prop_1\ \seq\ \prop_2)$: always avoid obstacles 1 and 2, eventually reach coffee and goal, and visit mug before coffee. This formula is insufficient to complete the true task (the ordering constraint between coffee and goal is not learned). This is because the optimal trajectories satisfying $\varphi_0$ and $\varphi^*$ are the same cost, i.e. both $\varphi_0$ and $\varphi^*$ are consistent with the demonstration and could have been returned, and $\varphi_0, \varphi^* \in \varphi_g$ (c.f. Sec. \ref{sec:theory}). Now, we also use the blue demonstration from Env. 1 (two examples total). Running Alg. \ref{alg:falsification} terminates in two iterations at $\Nsat = 14$ with the formulas $\varphi_1 = \avoid \prop_4 \wedge \avoid \prop_5 \wedge \eventually_{[0, T_j-1]} \prop_1 \wedge \eventually_{[0, T_j-1]} \prop_2 \wedge \eventually_{[0, T_j-1]} \prop_3$ (which enforces that the mug, coffee, and goal must be eventually visited, but in any order, while avoiding obstacles) and $\varphi_2 = \varphi^*$. Since the demonstration in Env. 1 doubles back to the coffee before going to goal, increasing its cost over first going to goal and then to coffee, the ordering constraint between the two is learnable. We also plot the generated counterexample (Fig. \ref{fig:transfer}, yellow), which achieves a lower cost, since $\varphi_1$ involves no ordering constraints. We use the learned formula to plan a path completing the task in a new environment (App. \ref{sec:app_experiments_transfer}). Overall, this example suggests we can use demonstrations in different environments to learn shared task structure and disambiguate between potential explanations.

\begin{figure}[h]
\centering
\vspace{-3pt}
\includegraphics[width=8.5cm]{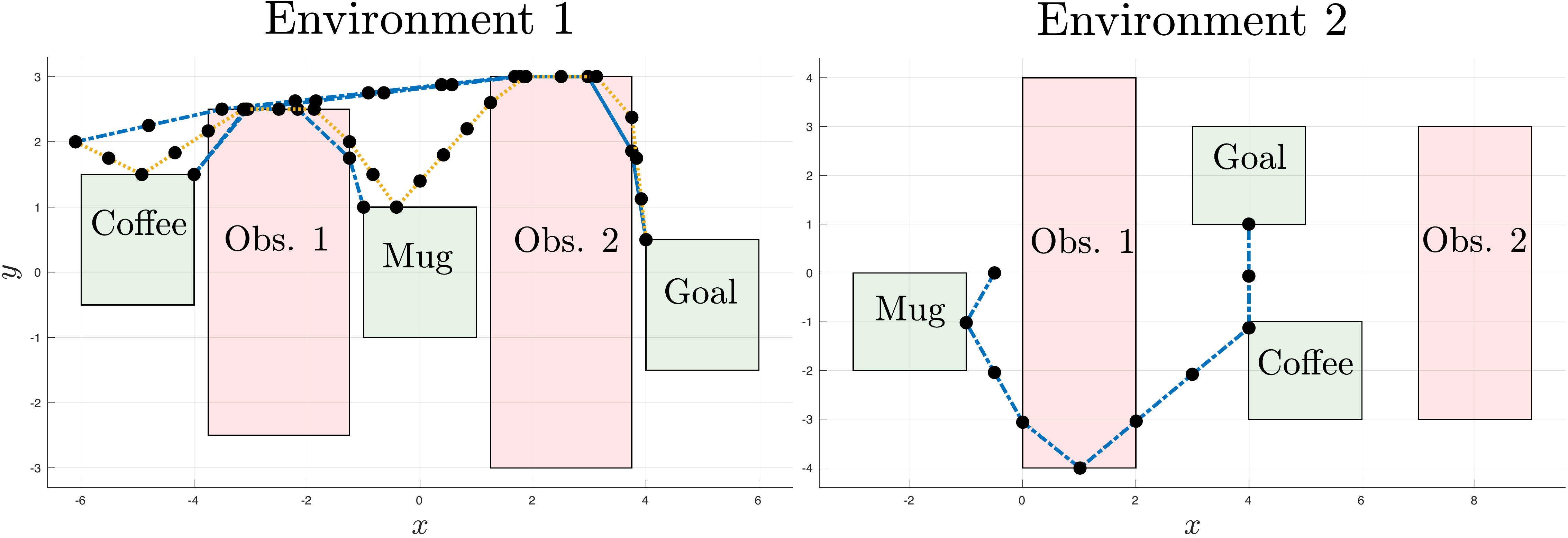}
\caption{Different environments (different $\params$) with shared task (same $\struct$).\vspace{-10pt}}\label{fig:transfer}
\vspace{-5pt}
\end{figure}

\begin{figure}[h]
\centering
\includegraphics[width=8.8cm]{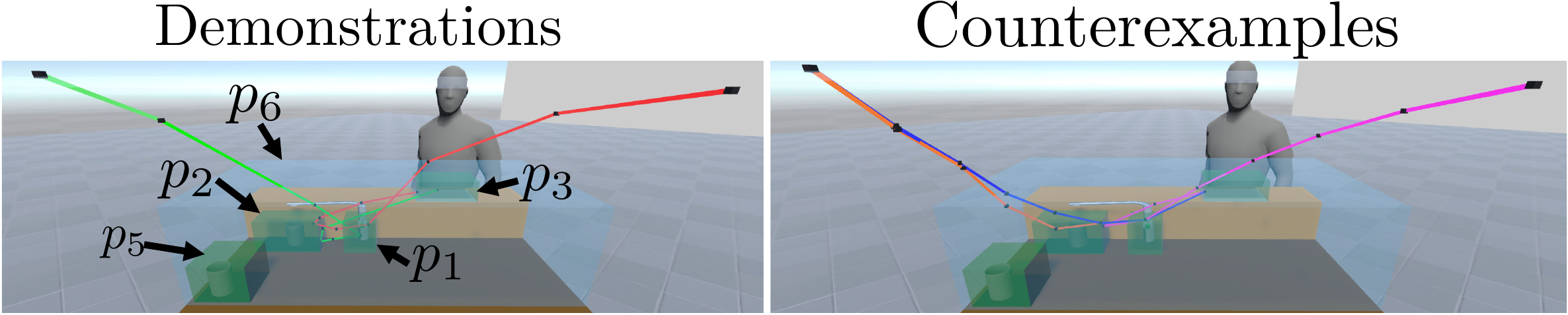}\vspace{1pt}
\caption{Demonstrations and counterexamples for the manipulation task.}\label{fig:arm_demos}
\vspace{-6pt}
\end{figure}

\vspace{-8pt}
\noindent\textbf{Multi-stage manipulation task}: We consider the setup in Figs. \ref{fig:arm_setup}, \ref{fig:arm_demos} of teaching a 7-DOF Kuka iiwa robot arm to prepare a drink: first move the end effector to the button on the faucet ($\prop_1$), then grasp the cup ($\prop_2$), then move the cup to the customer ($\prop_3$), all while avoiding obstacles. After grasping the cup, an end-effector pose constraint $(\alpha, \beta, \gamma) \in \safeset_4(\params_4)$ ($\prop_4$) must be obeyed. We add two ``distractor" APs: a different cup ($\prop_5$) and a region ($\prop_6$) where the robot can hand off the cup. We also modify the grammar to include the sequence operator $\seq$, (defined as before), and add an ``after" operator $\varphi_1\ \mathcal{T}\ \varphi_2 \doteq \always_{[0, T_j-1]}(\varphi_2 \rightarrow \always_{[0, T_j-1]} \varphi_1)$, that is, $\varphi_1$ must hold after and including the first timestep where $\varphi_2$ holds. The true formula is: $\varphi^* = (\prop_1\ \seq\ \prop_2) \wedge (\prop_2\ \seq\ \prop_3) \wedge \eventually_{[0, T_j-1]} \prop_3 \wedge (\prop_4\ \mathcal{T}\ \prop_2)$. We use a kinematic arm model: $j_{t+1}^i = j_t^i + u_t^i$, $i=1,\ldots, 7$, where $\Vert u_t \Vert_2^2 \le 1$ for all $t$. Two suboptimal human demonstrations ($\delta=0.7$) optimizing $c(\trajxu) = \sum_{t=1}^{T-1} \Vert j_{t+1} - j_t\Vert_2^2$ are recorded in virtual reality. We assume we have nominal estimates of the AP regions $\safeset_i(\params_{i,\textrm{nom}})$ (i.e. from a vision system), and we want to learn the $\struct$ and $\params$ of $\varphi^*$. 

\vspace{-1pt}
We run Alg. \ref{alg:falsification} with the 1-SO constraints \eqref{eq:specopt}, and encode the nominal $\params_i$ by enforcing that $\Theta_{i}^p = \{\params_i \mid \Vert \params_i - \params_{i,\textrm{nom}} \Vert_1 \le 0.05 \}$. At $\Nsat = 11$, the inner loop runs for 3 iterations (each taking 30 minutes on an i7-7700K processor), returning candidates $\varphi_1 = (\prop_1\seq\prop_3) \wedge (\prop_2\seq\prop_3) \wedge (\eventually_{[0,T_j-1]} \prop_3) \wedge (\prop_4 \mathcal{T} \prop_3)$, $\varphi_2 = (\prop_1\seq\prop_3) \wedge (\prop_2\seq\prop_3) \wedge (\eventually_{[0,T_j-1]} \prop_3) \wedge (\prop_4 \mathcal{T} \prop_2)$, and $\varphi_3 = \varphi^*$. $\varphi_1$ says that before going to the customer, the robot has to visit the button and cup in any order, and then must satisfy the pose constraint after visiting the cup. $\varphi_2$ has the meaning of $\varphi^*$, except the robot can go to the button or cup in any order. Note that $\varphi_3$ is a stronger formula than $\varphi_2$, and $\varphi_2$ than $\varphi_1$; this is a natural result of the falsification loop, which returns incomparable or stronger formulas with more iterations, as the counterexamples rule out weaker or equivalent formulas. Also note that the distractor APs don't feature in the learned formulas, even though both demonstrations pass through $\prop_6$. This happens for two reasons: we increase $\Nsat$ incrementally and there was no room to include distractor objects in the formula (since spec-optimality may enforce that $\prop_1$-$\prop_3$ appear in the formula), and even if $\Nsat$ were not minimal, $\prop_6$ would not be guaranteed to show up, since visiting $\prop_6$ does not increase the trajectory cost. 

We plot the counterexamples in Fig. \ref{fig:arm_demos}: blue/purple are from iteration 1; orange is from iteration 2. They save cost by violating the ordering and pose constraints: from the left start state, the robot can save cost if it visits the cup before the button (blue, orange trajectories), and loosening the pose constraint can reduce joint space cost (orange, purple trajectories). The right demonstration produces no counterexample in iteration 2, as it is optimal for this formula (changing the order does not lower the optimal cost). For the learned $\params$, $\params_i = \params_{i,\textrm{nom}}$ except for $\prop_2$, $\prop_3$, where the box shrinks slightly from the nominal; this is because by tightening the box, a Lagrange multiplier can be increased to reduce the KKT residual. We use the learned $\params$, $\struct$ to plan formula-satisfying trajectories from new start states (see App. \ref{sec:app_experiments}). Overall, this example suggests that Alg. \ref{alg:falsification} can recover $\params$ and $\struct$ on a high-dimensional problem and ignore distractor APs, despite demonstration suboptimality.

\setlength{\textfloatsep}{0pt}
\begin{figure}[h]
\centering
\includegraphics[width=\linewidth]{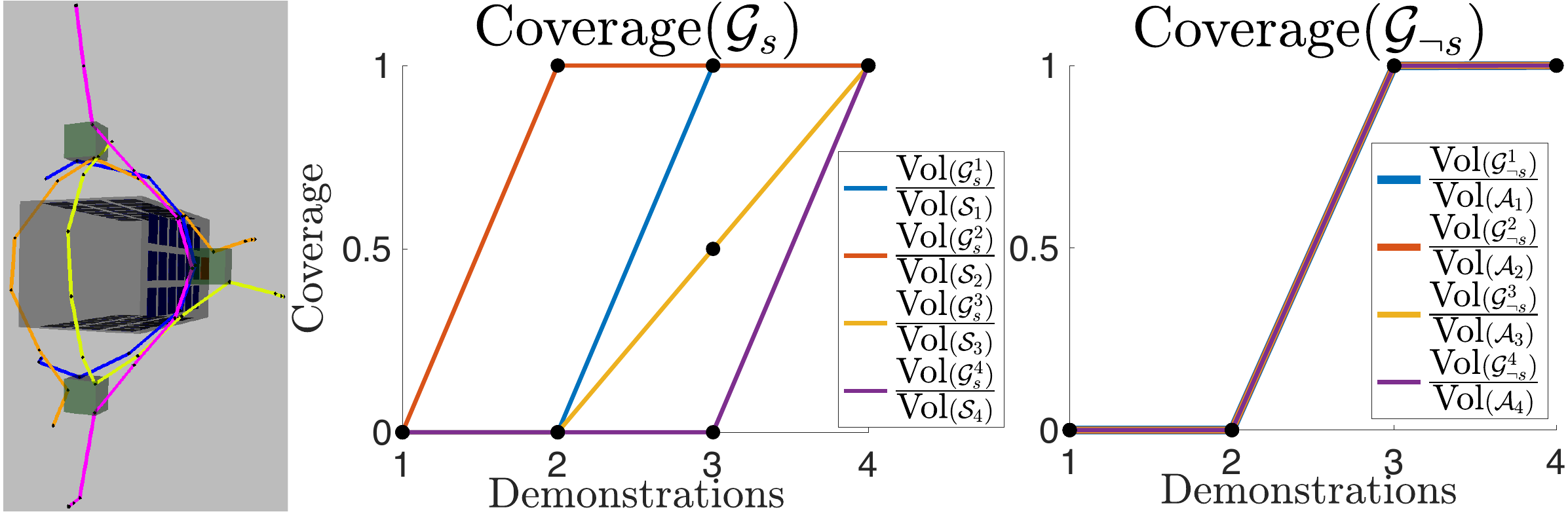}\vspace{-5pt}
\caption{Quadrotor surveillance demonstrations and learning curves.}\label{fig:quad_demos}
\end{figure}

\noindent\textbf{Multi-stage quadrotor surveillance}: We demonstrate that we can jointly learn $\params$, $\struct$, and $\costparams$ in one shot on a 12D nonlinear quadrotor system (see App. \ref{sec:app_experiments_quad}). We are given four demonstrations of a quadrotor surveilling a building (Fig. \ref{fig:quad_demos}): it needs to visit three regions of interest (Fig. \ref{fig:quad_demos}, green) while not colliding with the building. All visitation constraints can be learned directly with 1-SO (see Rem. \ref{rem:surveillance}) and collision-avoidance can also be learned with 1-SO, with enough demonstrations. The true formula is $\varphi^* = \eventually_{[0, T_j-1]} \prop_1 \wedge \eventually_{[0, T_j-1]} \prop_2 \wedge \eventually_{[0, T_j-1]} \prop_3 \wedge \always_{[0,T_j-1]} \neg \prop_4$, where $\prop_1$-$\prop_3$ represent the regions of interest and $\prop_4$ is the building. We aim to learn $\params_i$ for the parameterization $\safeset_i(\params_i) = \{ [I_{3\times 3}, -I_{3\times 3}]^\top [x, y, z]^\top \le \params_i\}$, assuming $\params_{4,6} = 0$ (the building is not hovering). The demonstrations minimize $c(\trajxu, \costparams) = \sum_{r \in R}\sum_{t=1}^{T-1} \gamma_r (r_{t+1} - r_t)^2$, where $R=\{x,y,z,\dot\alpha,\dot\beta,\dot\gamma\}$ and $\gamma_r = 1$, i.e. equal penalties to path length and angular acceleration. We assume $\gamma_r \in [0.1, 1]$ and is unknown: we want to learn the cost weights for each state. 

Solving Prob. \ref{prob:globopt} with 1-SO conditions (at $\Nsat = 12$) takes 44 minutes and recovers $\params$, $\struct$, and $\costparams$ in one shot. To evaluate the learned $\params$, we show in Fig. \ref{fig:quad_demos} that the coverage of the $\guarsafe^i$ and $\guarunsafe^i$ for each $\prop_i$ (computed by fixing the learned $\struct$ and running Prob. \ref{prob:volextract}) monotonically increases with more data. In terms of recovered $\struct$, with one demonstration, we return $\varphi_1 = \eventually_{[0, T_j-1]} \prop_2 \wedge \eventually_{[0, T_j-1]} \prop_3 \wedge \eventually_{[0, T_j-1]} \prop_4 \wedge \always_{[0,T_j-1]} \neg \prop_1$. This highlights the fact that since we are not provided labels, there is an inherent ambiguity of how to label the regions of interest (i.e. $\prop_i$, $i =1,\ldots,3$ can be associated with any of the green boxes in Fig. \ref{fig:quad_demos} and be consistent). Also, one of the regions of interest in $\varphi$ gets labeled as the obstacle (i.e. $\prop_1$ and $\prop_4$ are swapped), since one demonstration is not enough to disambiguate which of the four $\prop_i$ should touch the ground. Note that this ambiguity can be eliminated if labels are provided (see App. \ref{sec:prior_knowledge}) or if more demonstrations are provided: for two and more demonstrations, we learn $\varphi_i = \varphi^*$, $i = 2,\ldots,4$. When using all four demonstrations, we recover the cost parameters $\costparams$ and structure $\struct$ exactly, i.e. $\ltlprop = \varphi^*$, and fixing the learned $\struct$ and running Prob. \ref{prob:volextract} returns $\guarsafe^i = \safeset_i$ and $\guarunsafe^i = \unsafeset_i$, for all $i$. The learned $\costparams$, $\struct$, and $\params$ are used to plan trajectories that efficiently complete the task for different initial and goal states (see App. \ref{sec:app_experiments_quad}). Overall, this example suggests that our method can jointly recover a consistent set of $\params$, $\struct$, and $\costparams$ for high-dimensional systems.

\section{Conclusion}

We present an method that learns LTL formulas with unknown atomic propositions and logical structure from only positive demonstrations, assuming the demonstrator is optimizing an uncertain cost function. We use both implicit (KKT) and explicit (algorithmically generated lower-cost trajectories) optimality conditions to reduce the hypothesis space of LTL specifications consistent with the demonstrations. In future work, we aim to robustify our method to mislabeled demonstrations, explicitly consider demonstration suboptimality arising from risk, and reduce our method's computation time. 

\section*{Acknowledgments}
\footnotesize The authors thank Daniel Neider for insightful discussions. This research was supported in part by an NDSEG fellowship, NSF grants IIS-1750489 and ECCS-1553873, and ONR grants N00014-17-1-2050 and N00014-18-1-2501.

\normalsize
\clearpage
\bibliographystyle{plainnat}
\bibliography{references}
\clearpage
\appendix

\begin{figure*}[b]
\hrulefill
\begin{equation}\label{eq:ltl_semantics}
	\hspace{-25pt}\begin{array}{>{\displaystyle}c >{\displaystyle}l >{\displaystyle}l}
		(\trajxu,t) \models p_i & \Leftrightarrow & \unkfeat_i(\feat_i(\state_t), \params_i) \le \bzero \\
		(\trajxu,t) \models \neg p_i & \Leftrightarrow & \neg((\trajxu,t) \models p_i ) \\
		(\trajxu,t) \models \varphi_1 \orltl \varphi_2 & \Leftrightarrow & (\trajxu,t)\models \varphi_1 \orltl (\trajxu,t)\models \varphi_2 \\
		(\trajxu,t) \models \varphi_1 \andltl \varphi_2 & \Leftrightarrow &  (\trajxu,t)\models \varphi_1 \andltl (\trajxu,t)\models \varphi_2 \\
		(\trajxu,t) \models \always_{[t_1,t_2]} \varphi & \Leftrightarrow & (t+t_1 \le T) \wedge (\forall \tilde t \in [t+t_1,\min(t+t_2,T)], (\trajxu, \tilde t) \models \varphi) \\
		(\trajxu,t) \models \varphi_1 \mathcal{U}_{[t_1,t_2]} \varphi_2 & \Leftrightarrow & \begin{split}(t+t_1 \le T) \wedge (\exists \tilde t \in [t+t_1,\min(t+t_2, T)] \textrm{ s.t. } (\trajxu,\tilde t) \models \varphi_2)\ \andltl \\(\forall \check t \in [t,\tilde t-1], (\trajxu,\check t) \models \varphi_1)\end{split} \\
		(\trajxu,t) \models \eventually_{[t_1,t_2]} \varphi & \Leftrightarrow & (t+t_1 \le T) \wedge (\exists \tilde t \in [t+t_1,\min(t+t_2,T)] \textrm{ s.t. } (\trajxu, \tilde t) \models \varphi) \\
	\end{array}\hspace{-40pt}
\end{equation}
\end{figure*}

In this supplementary material, we present a detailed description of the LTL semantics that we use (Appendix \ref{sec:app_semantics}), further explanation and motivation for spec-optimality, a discrete optimality condition that we leverage (Appendix \ref{sec:spec_opt}), an expanded version of our theoretical analysis (Appendix \ref{sec:app_theory}), including proofs and additional comments, minor extensions and expanded details on our method (Appendix \ref{sec:app_method}), and some additional experimental details (Appendices \ref{sec:app_experiments_transfer}-\ref{sec:app_experiments_quad}).

\appendices

\section{LTL Semantics}\label{sec:app_semantics}

We denote the satisfaction of a formula $\ltl$ on a finite-duration trajectory $\trajxu$ of total duration $T$, evaluated at time $t \in \{1, 2, \ldots, T\}$, as $(\trajxu, t) \models \varphi$. Then, the formula satisfaction is defined recursively in \eqref{eq:ltl_semantics}. We emphasize that since we consider discrete-time trajectories, a time interval $[t_1, t_2]$ is evaluated only on integer time instants $\{t_1, t_1+1, \ldots, t_2\}$; this is made concrete in \eqref{eq:ltl_semantics}. We also provide the following intuitive description of the LTL operators:

\begin{itemize}
	\item The ``or" operator $\varphi_1 \vee \varphi_2$ denotes a disjunction between formulas $\varphi_1$ and $\varphi_2$
	\item The ``and" operator $\varphi_1 \wedge \varphi_2$ denotes a conjunction between formulas $\varphi_1$ and $\varphi_2$
	\item The ``bounded-time always" operator $\always_{[t_1,t_2]} \varphi$ denotes that a formula $\varphi$ always holds over the interval $[t_1, t_2]$
	\item The ``bounded-time until" operator $\varphi_1\ \mathcal{U}_{[t_1,t_2]}\ \varphi_2$ denotes that a formula $\varphi_2$ eventually holds during the interval $[t_1, t_2]$, and for all timesteps prior to that, $\varphi_1$ must hold.
	\item The ``bounded-time eventually" operator $\eventually_{[t_1,t_2]} \varphi$ denotes that a formula $\varphi$ eventually has to hold during the interval $[t_1, t_2]$.
\end{itemize}

\section{Details on Spec-Optimality}\label{sec:spec_opt}
For ease of reading, we have copied the definition of spec-optimality from the main body.
\begin{definition}[Spec-optimality]\label{def:app_specopt}
	A demonstration $\demj$ is \textit{$\mu$-spec-optimal ($\mu$-SO)}, where $\mu \in \mathbb{Z}_+$, if for every index set $\iota \doteq \{(i_1,t_1),...,(i_\mu,t_\mu) \}$ in $\mathcal{I} \doteq \{ \iota \mid i_m \in \{1,...,\Nap\}, t_m \in \{1,..., T_j\}, m=1,...,\mu\}$, at least one of the following holds:
	\begin{itemize}
		\item $\demj$ is locally-optimal after removing the constraints associated with $\prop_{i_m}$ on $\cstate_{t_m}^j$, for all $(i_m, t_m) \in \iota$.
		\item For each index $(i_m, t_m) \in \iota$, the formula is not satisfied for a perturbed $\Bool$, denoted $\hat\Bool$, where $\hat\Booll_{i_m,t_m}(\params_{i_m}) = \neg \Booll_{i_m,t_m}(\params_{i_m})$, for all $m=1,\ldots,\mu$, and $\hat\Booll_{i',t'}(\params_{i'}) = \Booll_{i',t'}(\params_{i'})$ for all $(i',t')\notin \iota$.
		\item $\demj$ is infeasible with respect to $\hat\Bool$.
	\end{itemize}
\end{definition}

Recall from the main body that spec-optimality aims to enforce a level of logical optimality. More specifically, spec-optimality is a measure of logical optimality, evaluated locally around a demonstration. To see this, note that the conditions in Def. \ref{def:app_specopt} are essentially checking how the local optimality of a demonstration changes as a result of local perturbations to the assignments of the discrete variables $\Bool$. The three conditions in Def. \ref{def:app_specopt} capture the three possibilities upon perturbing $\Bool$: the demonstration could become infeasible if $\Bool$ is perturbed (this is what the third condition checks), the demonstration could remain feasible but local optimality may not change (this is what the first condition checks), or the demonstration could remain feasible and no longer be locally-optimal (this is what the second condition checks). By enforcing that a demonstration is spec-optimal with respect to the formula being satisfied, we enforce that this last possibility (feasible but not locally-optimal) never occurs. We would want to enforce this, for instance, if the demonstration is assumed to be globally-optimal for the true LTL formula, because there should be no alternative assignment of $\Bool$ which admits a direction in which the demonstration cost can be improved. In terms of learning the LTL formula, demonstration spec-optimality is a useful condition because it prunes trivially-satisfiable formulas from the search space, spec-optimality necessarily holds for a demonstration to be globally-optimal (Lem. \ref{lem:app_specglobopt}), and it can be compactly enforced within a MILP, making it an efficient surrogate for enforcing global optimality. Generally, without spec-optimality, the falsification loop in Alg. \ref{alg:falsification} will need to eliminate more formulas on the way to finding a formula which makes the demonstrations globally-optimal. As a final note, we can interpret $\mu$ as a tuning knob for shifting the computation between the falsification loop and Prob. \ref{prob:globopt}; imposing a larger $\mu$ can potentially rule out more formulas at the cost of adding additional constraints and decision variables to Problem \ref{prob:globopt}.

To show how these conditions can be used on a concrete example, consider again the problem visualized in Fig. \ref{fig:venn_diagram}. In this problem, $\params_1, \params_2$ are known and we are given two kinematic demonstrations minimizing path length under input constraints, formula $\varphi = (\neg p_2\ \mathcal{U}_{[0, T_j-1]}\ p_1)\wedge \eventually_{[0,T_j-1]} p_2 $, and start/goal constraints. We also introduce a different candidate formula $\hat\varphi = \eventually_{[0,T_j-1]} p_1 \vee \eventually_{[0,T_j-1]} p_2$. For this example, consider only the blue demonstration. If the demonstration is globally-optimal, then $\hat\varphi$ is inconsistent, because to improve the trajectory cost, the demonstrator would choose to visit only one of $\mathcal{S}_1$ or $\mathcal{S}_2$, not both. 

We will show how spec-optimality can be used to distinguish between $\varphi$ and $\hat\varphi$; specifically, we show the demonstration is 1-SO with respect to $\varphi$ but not for $\hat\varphi$. For both $\varphi$ and $\hat\varphi$, $\mathcal{I} = \{(1, 1), \ldots, (1, 5), (2, 1), \ldots, (2, 5)\}$. Let's consider $\varphi$ first. For $\iota \in \{(1, 1), (2, 1), (2, 2), (1, 3),$ $ (2, 3), (1, 4), (1, 5), (2, 5)\}$, the third condition in Def. \ref{def:app_specopt} will hold, since at these timesteps, the demonstration is not on the boundary of the paired AP. For $\iota \in \{(1, 2), (2, 4)\}$, the second condition in Def. \ref{def:app_specopt} will hold, since perturbing $\Bool$ at either of these timesteps (from $\Booll_{1,2}(\params_1) = 1$ to $0$ or from $\Booll_{2,4}(\params_2) = 1$ to $0$) will cause $\varphi$ to be not satisfied. Thus, the demonstration is spec-optimal with respect to $\varphi$. On the other hand, for $\hat\varphi$, again for $\iota \in \{(1, 1), (2, 1), (2, 2), (1, 3),$ $ (2, 3), (1, 4), (1, 5), (2, 5)\}$, the third condition in Def. \ref{def:app_specopt} will hold. However, none of the three conditions will hold for $\iota \in \{(1, 2), (2, 4)\}$, since the demonstration will not be locally-optimal upon relaxing the constraints for either $\prop_1$ or $\prop_2$, and since $\hat\varphi$ only enforces that either one of $\mathcal{S}_1$ or $\mathcal{S}_2$ are visited, $\hat\varphi$ is still satisfied if either $\Booll_{1,2}(\params_1)$ or $\Booll_{2,4}(\params_2)$ is flipped to 0. Hence, the demonstration is not spec-optimal with respect to $\hat\varphi$.

\section{Theoretical Analysis (Expanded)}\label{sec:app_theory}

In this section, we prove that our method is complete under some assumptions, without (Thm. \ref{thm:app_completeness}) or with (Cor. \ref{thm:app_specoptcompleteness}) spec-optimality, and that we can compute guaranteed conservative estimates of $\safeset_i$/$\unsafeset_i$ (Thm. \ref{thm:app_innerapprox}). Finally, we show that the stronger the assumptions on the demonstrator, the smaller the set of consistent formulas (the less ill-posed the problem) (Thm. \ref{thm:app_distinguishability}). We copy the theorem statements and assumptions from the main body for ease of reading.

\subsection{Correctness and conservativeness}

\begin{assumption}\label{ass:app_complete}
	Prob. \ref{prob:fwdprob_ctrexample} is solved with a complete planner.
\end{assumption}
\begin{assumption}\label{ass:app_locopt}
	Each demonstration is locally-optimal (i.e. satisfies the KKT conditions) for fixed boolean variables.
\end{assumption}
\begin{assumption}\label{ass:app_feasible}
	The true parameters $\params$, $\struct$, and $\costparams$ are in the hypothesis space of Prob. \ref{prob:globopt}: $\params \in \Theta^p$, $\struct \in \Theta^s$, $\costparams \in \Theta^c$.
\end{assumption}

We will use these assumptions to show that when the cost function parameters $\params$ are known, our falsification loop in Alg. \ref{alg:falsification} is guaranteed to return a consistent formula; that is, it makes the demonstrations globally-optimal.

\begin{theorem}[Completeness \& consistency, unknown $\struct$, $\params$]\label{thm:app_completeness}
	Under Assumptions \ref{ass:app_complete}-\ref{ass:app_feasible}, Alg. \ref{alg:falsification} is guaranteed to return a formula $\ltl$ such that 1) $\demj \models \ltl$ and 2) $\demj$ is globally-optimal under $\ltl$, for all $j$, 3) if such a formula exists and is representable by the provided grammar.
\end{theorem}
\begin{proof}
	To see the first point - that Alg. \ref{alg:falsification} returns $\ltlprop$ such that $\demj \models \ltlprop$ for all $j$, note that in Prob. \ref{prob:globopt}, the constraints \eqref{eq:sat1}-\eqref{eq:sat3} on the satisfaction matrices $\satmatrix_j^\textrm{dem}$ encode that each demonstration is feasible for the choice of $\params$ and $\struct$; hence, the output of Prob. \ref{prob:globopt} will return a feasible $\ltlprop$. As Alg. \ref{alg:falsification} will eventually return some $\ltlprop$ which is an output of Prob. \ref{prob:globopt}, the $\ltlprop$ that is ultimately returned is feasible.
	
	Next, to see the second point - that the ultimately returned $\ltlprop$ makes each $\demj$ globally-optimal, note that at some iteration of the inner loop, if Prob. \ref{prob:fwdprob_ctrexample} is feasible and its solution algorithm is complete (Assumption \ref{ass:app_complete}), it will return a trajectory which is lower-cost than the demonstration and satisfies $\ltlprop$. Note that Prob. \ref{prob:fwdprob_ctrexample} will always be feasible, since Prob. \ref{prob:globopt} returns $\params, \struct$ for which the demonstration is feasible, and the feasible set of Prob. \ref{prob:fwdprob_ctrexample} contains the demonstration. The falsification loop will continue until Prob. \ref{prob:fwdprob_ctrexample} cannot produce a trajectory of strictly lower cost for each demonstration; this is equivalent to ensuring that each demonstration is globally optimal for the $\ltlprop$.
	
	To see the last point, we note that if there exists a formula $\ltlprop$ which satisfies the demonstrations, it is among the feasible set of possible outputs of Alg. \ref{alg:falsification}; that is, the representation of LTL formulas, $\Dag$, is complete (c.f. Lemma 1 in \cite{daniel}).
\end{proof}

We will further show that the formula returned by Alg. \ref{alg:falsification} is the shortest formula which is consistent with the demonstrations; this is due to $N_\textrm{DAG}$ only being incremented upon infeasibility of a smaller $N_\textrm{DAG}$ to explain the demonstrations.
\begin{corollary}[Shortest formula]
	Let $\Nmin$ be the minimal size DAG for which there exists $(\params, \struct)$ such that $\demj \models \ltl$ for all $j$. Under Assumptions \ref{ass:app_complete}-\ref{ass:app_feasible}, Alg. \ref{alg:falsification} is guaranteed to return a DAG of length $\Nmin$.
\end{corollary}
\begin{proof}
	The result follows since Algorithm \ref{alg:falsification} increases $\Nsat$ incrementally (in the outer loop) until some $\ltlprop$ is returned which makes all of the demonstrations feasible and globally-optimal, and each inner iteration of Algorithm \ref{alg:falsification} is guaranteed to find a consistent $\ltlprop$ if one exists (c.f. Theorem \ref{thm:app_completeness}).
\end{proof}

As leveraging a notion of discrete optimality is crucial to reduce the search space of LTL formulas, we show that all globally-optimal demonstrations must also be $\mu$-spec-optimal for the true specification, for any positive integer $\mu$.
\begin{lemma}\label{lem:app_specglobopt}
	All globally-optimal trajectories are $\mu$-SO.
\end{lemma}
\begin{proof}
We show that it is not possible for a demonstration $\demj$ to be globally-optimal while failing to satisfy (a), (b), and (c). If the constraints corresponding to $\prop_{i_m}$ at $\cstate_{t_m}^j$ are relaxed, for some $\{(i_m, t_m)\}_{m=1}^\mu$, then $\demj$ can either remain locally-optimal (which means (a) is satisfied, and happens if all the constraints are inactive or redundant) or become not locally-optimal. If $\demj$ becomes not locally-optimal for the relaxed problem (i.e. (a) is not satisfied), then at least one of the original constraints is active, implying $\bigvee_{m=1}^\mu \big( G_{i_m}(\cstate_{t_m}^j) = 0 \big)$. In this case, one of the following holds: either (1) each $\cstate_{t_m}^j$ lies on its constraint boundary: $\bigwedge_{m=1}^\mu \big( G_{i_m}(\cstate_{t_m}^j) = 0 \big)$, or (2) at least one $\cstate_{t_m}$ does not lie on its constraint boundary. If (2) holds, then $\demj$ must be infeasible for $\hat \Bool$, so (c) must be satisfied. If (1) holds, then $\demj$ is both feasible for $\hat\Bool$ and not locally-optimal with respect to the relaxed constraints. Then, there exists some trajectory $\hattrajxu$ such that $c(\hattrajxu) < c(\demj)$, and for at least one $m$ in $1,\ldots,\mu$, $G_{i_m}(\hat\cstate_{t_m}^j) > 0$, where $\hat\cstate_{t_m}^j$ is the constraint state at time $t_m$ on $\hattrajxu$. $\hattrajxu$ cannot be feasible with respect to the true specification, since it makes $\demj$ not globally-optimal, so in this case (b) must hold.
\end{proof}

Using the previous lemma, we can show that modifying Alg. \ref{alg:falsification} to additionally impose the spec-optimality conditions in Prob. \ref{prob:globopt} still enjoys the completeness properties discussed in Theorem \ref{thm:app_completeness}, while also in general reducing the number of falsification iterations needed as a result of the reduced search space.
\begin{corollary}[Alg. \ref{alg:falsification} with spec-optimality]\label{thm:app_specoptcompleteness}
	By modifying Alg. \ref{alg:falsification} so that Prob. \ref{prob:globopt} uses constraints \eqref{eq:specopt}, Alg. \ref{alg:falsification} still returns a consistent solution $\ltlprop$ if one exists, i.e. each $\demj$ is feasible and globally optimal for each $\ltlprop$.
\end{corollary}
\begin{proof}
	The result follows from completeness of Alg. \ref{alg:falsification} (c.f. Theorem \ref{thm:app_completeness}) and Lemma \ref{lem:app_specglobopt}: adding \eqref{eq:specopt_1}-\eqref{eq:specopt_3} enforces that $\demj$ are spec-optimal, and via Lemma \ref{lem:app_specglobopt}, $\demj$, which is a globally-optimal demonstration, must also be spec-optimal. Hence, imposing constraints \eqref{eq:specopt_1}-\eqref{eq:specopt_3} is consistent with the demonstration.
\end{proof}

Next, we show how the consistency properties extend to the case of unknown cost function, if Alg. \ref{alg:falsification_cost} returns a solution, which it is not guaranteed to do in finite time.
\begin{corollary}[Consistency, unknown $\costparams$]\label{thm:app_completeness_costunc}
	Under Assumptions \ref{ass:app_complete}-\ref{ass:app_feasible}, if Alg. \ref{alg:falsification_cost} terminates in finite time, it returns a formula $\ltl$ such that 1) $\demj \models \ltl$ and 2) $\demj$ is globally-optimal with respect to $\costparams$ under the constraints of $\ltl$, for all $j$, 3) if such a formula exists and is representable by the provided grammar.
\end{corollary}
\begin{proof}
	Note that Alg. \ref{alg:falsification_cost} is simply Alg. \ref{alg:falsification} with an outer loop where potential cost parameters $\costparams$ are chosen. From Theorem \ref{thm:app_completeness}, we know that under Assumptions \ref{ass:app_complete}-\ref{ass:app_locopt}, for the true cost parameter $\costparams$, Alg. \ref{alg:falsification} is guaranteed to return $\params$ and $\struct$ which make the demonstrations globally-optimal under $\costparams$. From Assumption \ref{ass:app_feasible} and the fact that the true parameters $\params$, $\struct$, and $\costparams$ will make the demonstrations globally-optimal, we know there exists at least one consistent set of parameters (the true parameters). Then, Alg. \ref{alg:falsification_cost} will eventually find a consistent solution (possibly the true parameters), as it iteratively runs Alg. \ref{alg:falsification} for all consistent $\costparams$.
\end{proof}

Finally, we show that for fixed LTL structure and cost function, querying and volume extraction (Problems \ref{prob:query} and \ref{prob:volextract}) are guaranteed to return conservative estimates of the true $\safeset_i$ or $\unsafeset_i$.
\begin{theorem}[Conservativeness for unknown $\params$]\label{thm:app_innerapprox}
	Suppose that $\struct$ and $\costparams$ are known, and $\params$ is unknown. Then, extracting $\guarsafe^i$ and $\guarunsafe^i$, as defined in \eqref{eq:guarsafe}-\eqref{eq:guarunsafe}, from the feasible set of Prob. \ref{prob:kkt_exact_ltl} projected onto $\Theta_i^p$ (denoted $\feas_i$), returns $\guarsafe^i \subseteq \safeset_i$ and $\guarunsafe^i \subseteq \unsafeset_i$, for all $i \in \{1, \ldots, \Nap\}$.
\end{theorem}
\begin{proof}
	We first prove that $\guarunsafe^i\subseteq \unsafeset_i$. Suppose that there exists $\cstate \in \guarunsafe^i$ such that $\cstate \notin \unsafeset_i$. Then by definition, for all $\params_i \in \feas_i$, $G_i(\cstate, \params_i) \ge 0$.  However, we know that all locally-optimal demonstrations satisfy the KKT conditions with respect to the true parameter $\theta_i^{p,*}$; hence, $\theta_i^{p,*} \in \feas$. Then, $\state \in \unsafeset(\theta_i^{p,*})$. Contradiction.	Similar logic holds for proving that $\guarsafe^i \subseteq \safeset_i$. Suppose that there exists $\state \in \guarsafe^i$ such that $\state \notin \safeset_i$. Then by definition, for all $\params_i \in \feas_i$, $G_i(\cstate, \params_i) \le 0$.  However, we know that all locally-optimal demonstrations satisfy the KKT conditions with respect to the true parameter $\theta_i^{p,*}$; hence, $\theta_i^{p,*} \in \feas_i$. Then, $\cstate \in \safeset_i(\theta_i^{p,*})$. Contradiction.
\end{proof}

\subsection{Learnability}

The goal of this theorem is to show that the stronger the assumptions on the demonstrator, the smaller the set of consistent formulas (the less ill-posed the problem).
\begin{theorem}[Distinguishability]\label{thm:app_distinguishability}
	For the consistent formula sets defined in Sec. \ref{sec:learnability}, we have $\formset_{g} \subseteq \formset_{\tilde\mu\textrm{-SO}} \subseteq \formset_{\hat\mu\textrm{-SO}} \subseteq \formset_f$, for $\tilde\mu > \hat\mu$.
\end{theorem}
\begin{proof}
	$\formset_g \subseteq \formset_{\tilde\mu\textrm{-SO}}$, since per Lemma \ref{lem:app_specglobopt}, all globally-optimal trajectories are $\tilde\mu$-SO. Thus, restricting Prob. \ref{prob:globopt} to enforce global optimality requires more constraints than restricting Prob. \ref{prob:globopt} to enforce $\tilde\mu$-SO. With more constraints, the feasible set of consistent formulas cannot be larger for global optimality. Similarly, as enforcing $\tilde\mu$-SO requires more constraints than enforcing $\hat\mu$-SO, the feasible set of consistent formulas cannot be larger for $\tilde\mu$-SO than for $\hat\mu$-SO.
	$\formset_{\mu\textrm{-SO}} \subseteq \formset_{f}$, since enforcing $\mu$-SO also enforces feasibility. Thus, restricting Prob. \ref{prob:globopt} to enforce $\mu$-SO requires more constraints than the standard Prob. \ref{prob:globopt}. With more constraints, the feasible set of consistent formulas cannot be larger for $\mu$-SO.
\end{proof}

\vspace{10pt}
\section{Method extensions and expanded details}\label{sec:app_method}

\subsection{Unknown cost algorithm}\label{sec:app_costalg}

In Algorithm \ref{alg:falsification_cost}, we formally write the analogue of the falsification approach in Alg. \ref{alg:falsification} when the cost function parameters $\costparams$ are unknown. The approach is the same as Algorithm \ref{alg:falsification}, apart from an additional outer while loop, where candidate $\costparams$ are selected. Upon the failure of a $\costparams$ to yield a consistent $\params$ and $\struct$, the $\costparams$ is added into a set of cost parameters for Problem \ref{prob:globopt} to avoid, $\Theta_\textrm{av}^c$. The avoidance condition can be implemented with integer constraints, i.e. $|\theta_i^c - \hat\theta_i^c| \ge \varepsilon_\textrm{av} - (1-z_\textrm{av}^i)$, $\sum_i z_\textrm{av}^i \ge 1$, for $i = 1,\ldots,|\theta_c|$. See Sec. \ref{sec:costparams} for more discussion.

\begin{algorithm}
\SetAlgoLined
\textbf{Input}: $\{\demj\}_{j=1}^{\numsafe}$, $\bar\safeset$, \textbf{Output}: $\hatstruct, \hatparams, \hatcostparams$\\
$\Nsat \leftarrow 0$, $\{\demunsafe\} \leftarrow \{\}$, $\Theta_\textrm{av}^c \leftarrow \{ \}$

\While{true}{
$\hatstruct$,\hspace{-1pt} $\hatparams$,\hspace{-1pt} $\hatcostparams$\hspace{-2pt} $\leftarrow$\hspace{-2pt} Problem \ref{prob:globopt}$'(\{\demj\}_{j=1}^{\numsafe}, \{\demunsafe\}, \Nsat, \Theta_\textrm{av}^c)$ \\
\While{$\neg$ consistent}{
$\Nsat \leftarrow \Nsat + 1$\\

\While{Problem \ref{prob:globopt} is feasible}{
$\hatstruct$, $\hatparams$ $\leftarrow$ Problem \ref{prob:globopt}$(\{\demj\}_{j=1}^{\numsafe}, \{\demunsafe\}, \Nsat, \hatcostparams)$ \\
\For{$j = 1$ to $\numsafe$}{
$\trajxu^j \leftarrow$ Problem \ref{prob:fwdprob_ctrexample}$(\demj)$ \\
\If{$c(\trajxu^j, \hatcostparams) < c(\demj, \hatcostparams)/(1+\delta)$}{
$\{\demunsafe\} \leftarrow \{\demunsafe\} \cup \trajxu$
}
}
\If{$\bigvee_{j=1}^{\numsafe} (c(\trajxu^j, \hatcostparams) < c(\demj, \hatcostparams)/(1+\delta))$}{
consistent $\leftarrow \top$; break\\}}}
\lIf{consistent}{ return} \lElse{$\Theta_\textrm{av}^c \leftarrow \Theta_\textrm{av}^c \cup \hatcostparams$; break} }
 \caption{Falsification, unknown cost function}\label{alg:falsification_cost}
\end{algorithm}

\subsection{Encoding prior knowledge}\label{sec:prior_knowledge}

\noindent \textbf{Known labels}: We have assumed that the demonstrations only include state/control trajectories and not the AP labels; this can lead to ambiguity as to which $\safeset$ should be assigned to which proposition $\prop_i$. For example, consider the example in Fig. \ref{fig:venn_diagram} (left), where the aim is to recover $\varphi(\params) = \eventually \safeset_1(\params_1) \vee \eventually \safeset_2(\params_2)$. The KKT conditions will imply that the demonstrator had to visit two boxes and their locations, but not if the left box should be labeled $\safeset_1$ or $\safeset_2$. However, in some settings it may be reasonable that the labels for each AP are provided, i.e. for an AP which requires a robot arm to grasp an object, we might have sensor data determining if the object has been grasped. In this case, we can incorporate this by simply constraining $\Bool_i^j(\params_i)$ to be the labels; this then removes the ambiguity mentioned earlier.

\noindent \textbf{Prior knowledge on $\params$}: In some settings, we may have a rough idea of $\params$, i.e. as noisy bounding boxes from a vision system. We might then want to avoid deviating from these nominal parameters, denoted $\params_\textrm{nom}$, or restrict $\params$ to some region around $\params_\textrm{nom}$, denoted $\Theta_{i,\textrm{nom}}$, subject to the KKT conditions holding. This can be done by adding $\sum_{j=1}^{\Nap}\Vert \params_i - \params_{i,\textrm{nom}} \Vert_1$ as an objective or $\params_{i,\textrm{nom}} \in \Theta_{i,\textrm{nom}}$ as a constraint to Prob. \ref{prob:kkt_exact_ltl}.

\subsection{Variants on the falsification loop}\label{sec:loopvariants}

Depending on the desired application, it may be useful to impose an ordering in which candidate structures $\struct$ are returned in line 4 of Alg. \ref{alg:falsification}. For example, the user may want to return the most restrictive formulas first (i.e. formulas with the smallest language), since more restrictive formulas are less likely to admit counterexamples (and hence the falsification should terminate in fewer iterations). On the other hand, the user may want to return the least restrictive formulas first, generating many invalid formulas in order to explicitly know what formulas do not satisfy the demonstrator's wishes. 

However, imposing an entailment-based ordering on the returned formulas is computationally challenging, as in general this will involve pairwise LTL entailment checks over a large set of possible LTL formulas, and each check is in PSPACE \cite{DemriS02}. Despite this, we can heuristically approximate this by assigning weights to each node type in the DAG based on their logical ``strength", such that each DAG with the same set of nodes has an overall weight $w = \sum_{u=1}^{\Nsat} \sum_{v=1}^{\Noper}w_{u,v}\parse_{u,v}$. For example, $\vee$ should be assigned a lower weight than $\andltl$, since $\vee$s can never restrict language size, while $\andltl$ can never grow it. Then, stronger/weaker formulas can be returned first by adding constraint $w \ge w_\textrm{thresh}$/$w \le w_\textrm{thresh}$, where $w_\textrm{thresh}$ is reduced/increased until a consistent formula is found. 

Note that multiple consistent formula structures can be also generated by adding a constraint for Prob. \ref{prob:globopt} to not return the same formula structure and continuing the falsification loop after the first consistent formula is found.

\section{Additional Experimental Details (Environment transfer)}\label{sec:app_experiments_transfer}
We use the learned $\struct$ to plan a trajectory which completes the high-level task (first going to the mug, then the coffee machine, then the goal, while always avoiding obstacles) in a novel environment map (with different AP parameters $\params$). Please refer to our accompanying video at \textbf{\textcolor{blue}{\url{https://youtu.be/cpUEcWCUMqc}}} for animations of the planned trajectories.

\begin{figure}[h]
\centering
\includegraphics[width=0.9\linewidth]{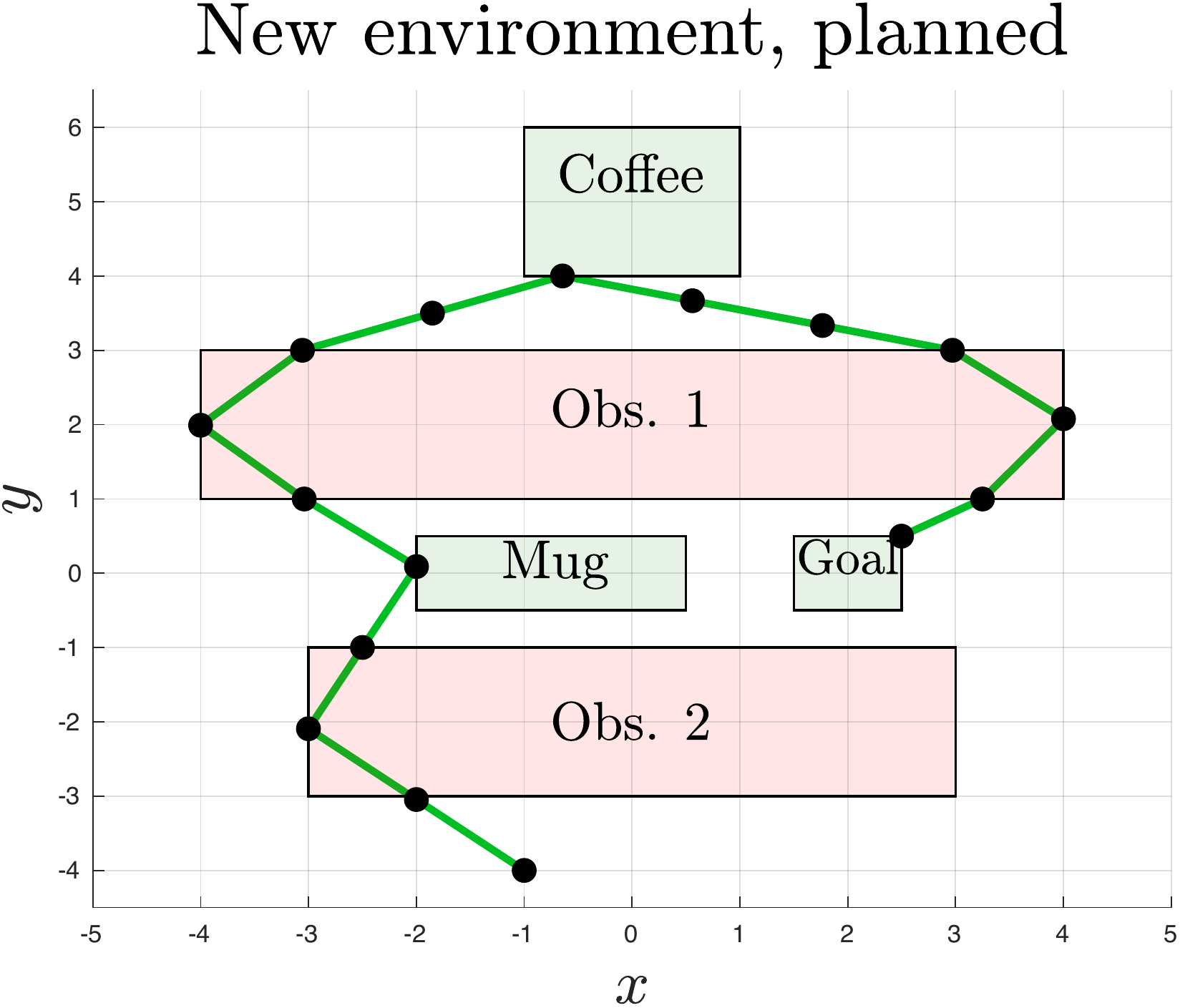}\vspace{-5pt}
\caption{Environment-transfer planned trajectory.}\label{fig:transfer_planned}
\end{figure}

\vspace{5pt}
\section{Additional Experimental Details (Manipulation)}\label{sec:app_experiments}
We use the learned $\params$ and $\struct$ to plan trajectories which complete the task from new initial conditions in the environment (Fig. \ref{fig:arm_planned}).  Please refer to our video at \textbf{\textcolor{blue}{\url{https://youtu.be/cpUEcWCUMqc}}} for animations of the planned trajectories. 

\begin{figure}[h]
\centering
\includegraphics[width=0.9\linewidth]{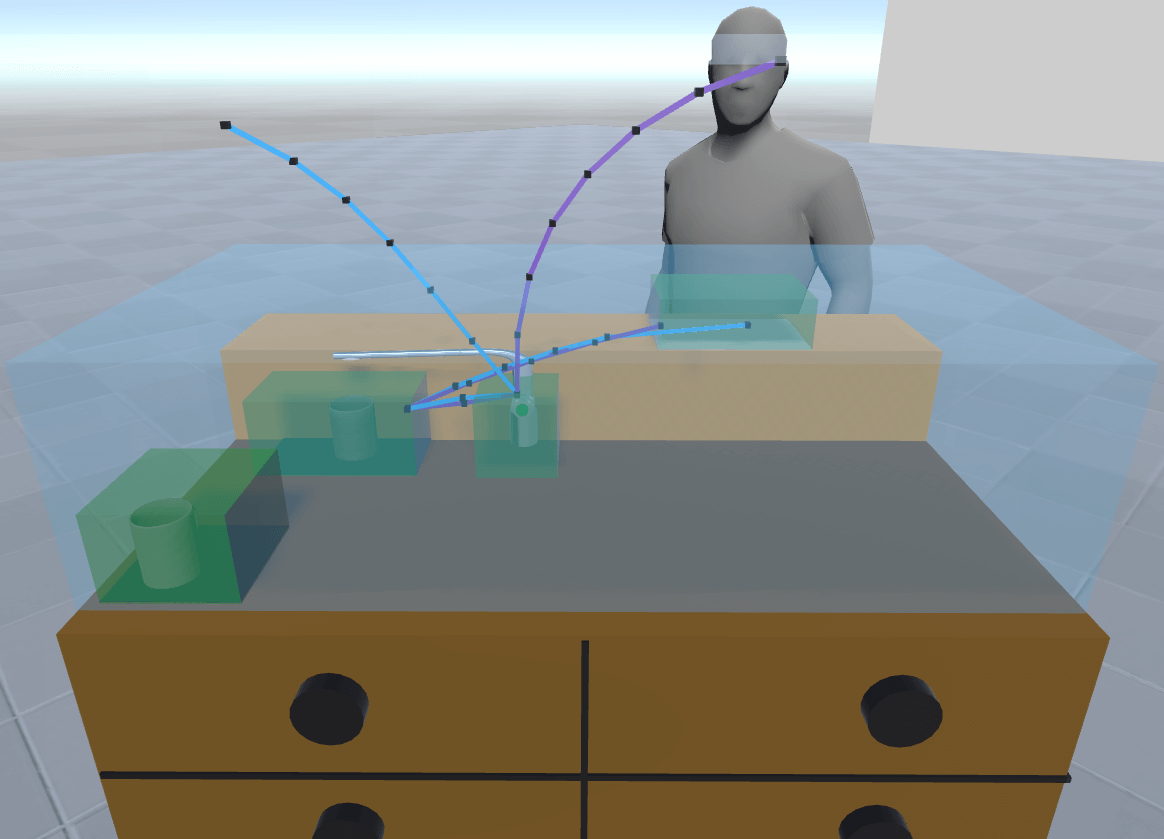}\vspace{-8pt}
\caption{Arm planned trajectories.\vspace{-10pt}}\label{fig:arm_planned}
\end{figure}
\vspace{-10pt}

\vspace{5pt}
\section{Additional Experimental Details (Quadrotor)}\label{sec:app_experiments_quad}

The system dynamics for the quadrotor \cite{quad_kth} are:
\begin{equation}
	\hspace{-5pt}\begin{bmatrix} \dot\chi \\ \dot y \\ \dot z \\ \dot\alpha \\ \dot\beta \\ \dot\gamma \\ \ddot \chi \\ \ddot y \\ \ddot z \\ \ddot \alpha \\ \ddot \beta \\ \ddot \gamma \end{bmatrix} = \begin{bmatrix} \dot\chi \\ \dot y \\ \dot z \\ \dot\beta \frac{\sin(\gamma)}{\cos(\beta)} + \dot\gamma \frac{\cos(\gamma)}{\cos(\beta)} \\ \beta \cos(\gamma) - \dot\gamma \sin(\gamma) \\ \dot\alpha + \dot\beta\sin(\gamma)\tan(\beta)+\dot\gamma\cos(\gamma)\tan(\beta) \\ -\frac{1}{m}[\sin(\gamma)\sin(\alpha) + \cos(\gamma)\cos(\alpha)\sin(\beta)]u_1 \\ -\frac{1}{m}[\cos(\alpha)\sin(\gamma) - \cos(\gamma)\sin(\alpha)\sin(\beta)]u_1 \\ g-\frac{1}{m}[\cos(\gamma)\cos(\beta)]u_1 \\ \frac{I_y-I_z}{I_x} \dot\beta \dot\gamma + \frac{1}{I_x}u_2 \\ \frac{I_z-I_x}{I_y} \dot\alpha \dot\gamma + \frac{1}{I_y}u_3\\ \frac{I_x-I_y}{I_z} \dot\alpha \dot\beta + \frac{1}{I_z}u_4\end{bmatrix},
\end{equation}
with control constraints $\Vert u_t \Vert_2 \le 10$. We time-discretize the dynamics by performing forward Euler integration with discretization time $\delta t = 1.2$ seconds. The 12D state is $x = [\chi, y, z, \alpha, \beta, \gamma, \dot x, \dot y, \dot z, \dot \alpha, \dot \beta, \dot \gamma]^\top$, and the relevant constants are $g = -9.81 \textrm{m}/\textrm{s}^2$, $m=1$kg, $I_x = 0.5\textrm{kg}\cdot\textrm{m}^2$, $I_y = 0.1\textrm{kg}\cdot\textrm{m}^2$, and $I_z = 0.3\textrm{kg}\cdot\textrm{m}^2$.

With the four demonstrations provided (see Sec. \ref{sec:results}), we learn $\costparams$, $\params$, and $\struct$, and obtain a representation of $\guarsafe^i$ and $\guarunsafe^i$, for all AP $p_i$. Using $\costparams$, $\struct$, and $\guarsafe^i$, we plan trajectories from new initial states to new goal states in the environment which are guaranteed to satisfy the true LTL formula; these trajectories are presented in Fig. \ref{fig:quad_planned}. Please refer to our accompanying video at \textbf{\textcolor{blue}{\url{https://youtu.be/cpUEcWCUMqc}}} for animations of the planned trajectories.

\begin{figure}[h]
\centering
\includegraphics[width=\linewidth]{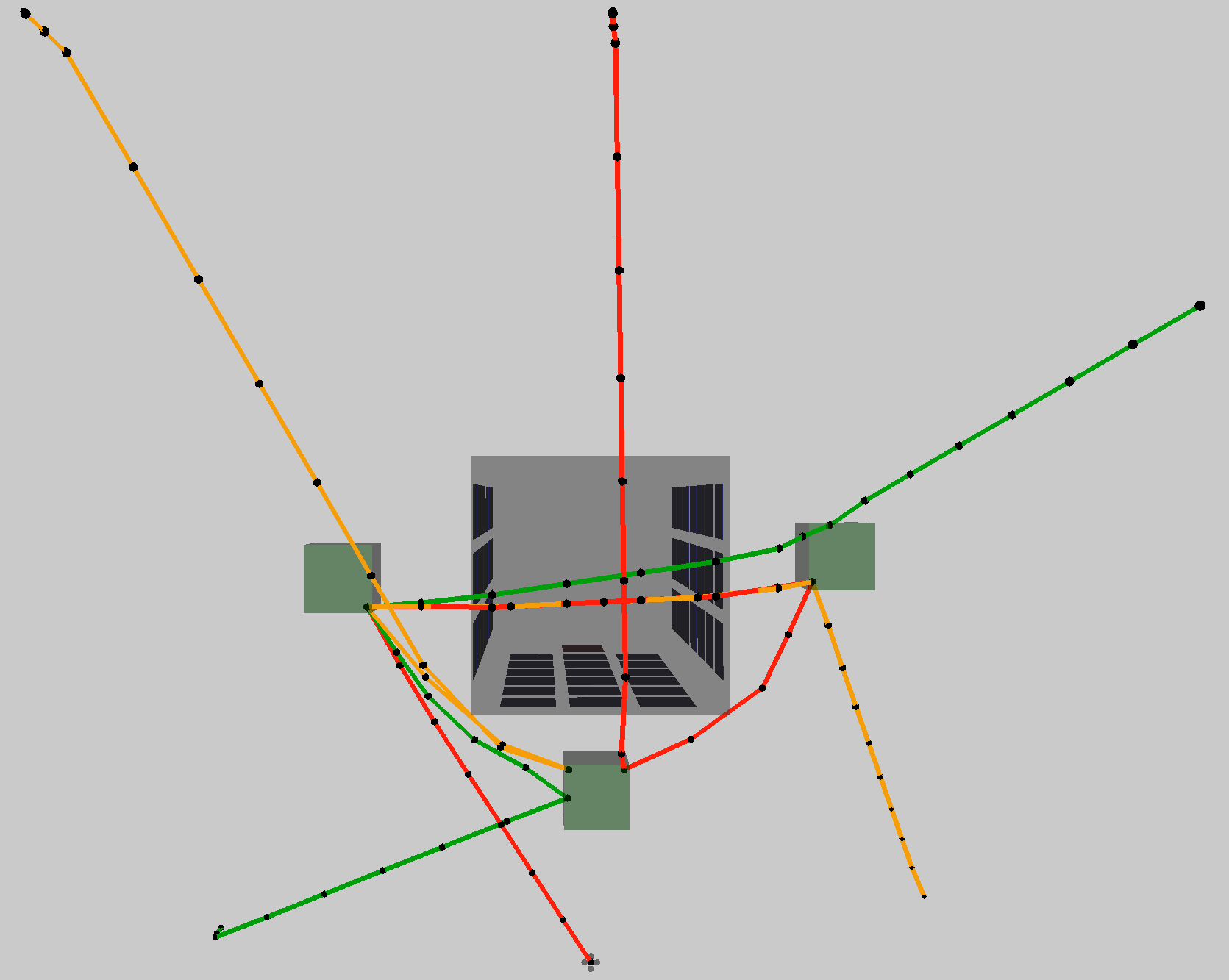}
\caption{Quadrotor planned trajectories.}\label{fig:quad_planned}
\end{figure}

\end{document}